\documentclass[twoside,11pt]{article}

\usepackage{arxiv}

\usepackage{amssymb}
\usepackage{graphicx}
\usepackage{amsmath}
\usepackage{xspace}
\usepackage{bm}
\usepackage{lineno}
\usepackage{subfigure}
\usepackage{afterpage}
\usepackage{comment}
\usepackage{afterpage}
\usepackage{indentfirst}
\usepackage{enumerate}
\usepackage[titlenumbered,ruled,lined,linesnumbered,algo2e]{algorithm2e}
\usepackage{color}
\usepackage{mathtools}
\usepackage{multirow}
\newtheorem{assumption}{Assumption}

\usepackage{hyperref}

\newcommand{\nbb}{\mathbb{N}}

\newcommand{\bw}{\mathbf{w}}
\newcommand{\bv}{\mathbf{v}}

\newcommand{\ibb}{\mathbb{I}}
\newcommand{\xcal}{\mathcal{X}}
\newcommand{\wcal}{\mathcal{W}}

\newcommand{\zcal}{\mathcal{Z}}

\newcommand{\ycal}{\mathcal{Y}}

\newcommand{\ebb}{\mathbb{E}}
\newcommand{\rbb}{\mathbb{R}}

\def\ga{\alpha}

\def\gep{\varepsilon}

\def\gl{\lambda}

\def\gO{\Omega}

\def\O{\mathcal{O}}

\def\R{\mathbb{R}}

\def\X{\mathcal{X}}
\def\Y{\mathcal{Y}}

\def\EX{{\mathbb{E}}}

\def\gl{\lambda}
\def\R{\mathbb{R}}
\def\bI{{\bf I}}

\def\bw{\mathbf{w}}

\def\bv{\mathbf{v}}

\newcommand\numberthis{\addtocounter{equation}{1}\tag{\theequation}}

\renewenvironment{proof}[1][Proof]{\par\noindent{\bf #1\ }}{\hfill$\square$\\[2mm]}

\numberwithin{equation}{section}

\jmlrheading{1}{2000}{1-48}{4/00}{10/00}{meila00a}{Yunwen Lei and Yiming Ying}

\ShortHeadings{}{}

\firstpageno{1}

\begin{document}

\title{Stochastic Proximal AUC Maximization}

\author{Yunwen Lei \email leiyw@sustc.edu.cn \\
       \addr Department of Computer Science and Engineering\\
       Southern University of Science and Technology\\
       Shenzhen, China
       \AND
       \name Yiming Ying\thanks{Corresponding author} \email yying@albany.edu \\
       \addr Department of Mathematics and Statistics\\
        State University of New York at Albany\\
       Albany, USA}

\editor{}

\maketitle

\begin{abstract}In this paper we consider the problem of maximizing the Area under the ROC curve (AUC)  which is a widely used performance metric in  imbalanced  classification and anomaly detection. Due to the pairwise nonlinearity of the objective function, classical SGD algorithms do not apply to the task of AUC maximization.    We propose a novel stochastic proximal algorithm for AUC maximization which is scalable to large scale streaming data. Our algorithm can accommodate general penalty terms and is easy to implement with favorable $\O(d)$ space and per-iteration time complexities. We establish a high-probability convergence rate $\O(1/\sqrt{T})$ for the general convex setting, and improve it to  a fast convergence rate $\O(1/T)$ for the cases of strongly convex regularizers and no regularization term (without strong convexity). Our proof does not need the uniform boundedness assumption on the loss function or the iterates which is more fidelity to the practice.  Finally, we perform extensive experiments over various benchmark data sets from real-world application domains which show the superior performance of our algorithm over the existing AUC maximization algorithms.
\end{abstract}


\section{Introduction}
Area under the ROC curve (AUC) \citep{Hanley} measures the probability for a randomly drawn positive instance to have a higher decision value than a randomly  sampled negative instance.
  It is a widely used metric for measuring the performance of machine learning algorithms in  imbalanced  classification and anomaly detection \citep{bradley1997use,fawcett2006introduction}. In particular, minimization of the rank loss in bipartite ranking  is equivalent to maximizing the AUC criterion \citep{agarwal2005generalization,guvenir2013ranking,kotlowski2011bipartite}.  At the same time, we are experiencing the fundamental change of the sheer size of commonly generated datasets where {\em streaming data} is  continuously arriving in a real time manner.  Hence, it is of practical importance to develop efficient optimization algorithms for maximizing the AUC score which is scalable to large-scale streaming datasets for real-time predictions.

Stochastic (proximal) gradient descent (SGD), also known as stochastic approximation or incremental gradient,  has become the workhorse in machine learning \citep{Bach,Bottou,Orabona,RS,rosasco2014convergence,ST}. It can be regarded as online learning \citep{cesa2006prediction,hazan2016introduction,shalev2012online} in the stochastic setting  where the individual data point is assumed to be drawn randomly from a (unknown) distribution.  These algorithms are iterative and incremental in nature and process each new sample (input) with a computationally cheap update, making them amenable for streaming data analysis. The working mechanism behind classical SGD algorithms is to perform gradient descent using unbiased (random) samples of the true gradient. In the sense, the objective function is required to be {\em linear} in the sampling distribution.  For example, in binary classification, let $\rho$ be a probability measure (sampling distribution) defined on input/output space $\X\times \Y$ with $\X\subseteq \R^d$ and $\Y = \{\pm 1\}.$ The linearity with respect to the sampling distribution $\rho$ in this case means that the objective function  (true risk) is the expectation of a {\em pointwise} loss function $\ell:\rbb^d\times \X \times \Y \to [0, \infty)$, i.e.   $$R(\bw)=\EX[\ell(\bw, x,y)] = \iint_{\X\times \Y} \ell(\bw, x , y)d\rho(x,y).$$   This linearity plays a pivotal role in studying the convergence  of SGD and deriving many of its appealing properties.

In contrast, the problem of AUC maximization involves the expectation of a {\em pairwise} loss function which depends on pairs of data points.  Consequently, the objective function in AUC maximization is {\em pairwise nonlinear} with respect to the sampling distribution $\rho$. To be more precise, recall \citep{Hanley,Clem} that  the AUC score of a function $h_\bw(x) = \bw^\top x$ is defined by
\begin{equation}\label{auc}
  \text{AUC}(\bw)  =\text{Pr}\{\bw^\top x\geq \bw^\top x'|y=+1,y'=-1\}=\ebb\big[\ibb_{[ \bw^\top x\geq \bw^\top x']}|y=+1,y'=-1\big],\numberthis
\end{equation}
where $\ebb[\cdot]$ is with respect to $(x,y)$ and $(x',y')$ independently drawn from $\rho$.
Since the indicator function $\ibb[\cdot]$ is discontinuous,  one often resorts to a convex surrogate loss $\ell:\rbb\mapsto\rbb^+$ and two common choices are the least square loss $\ell(a)=(1-a)^2$ and the hinge loss
$\ell(a)=\max\{0,1-a\}$. In this paper, we consider the least square loss since it is statistically consistent with AUC while the hinge loss is not \citep{gao2015consistency}.  Hence, we have \begin{equation}\label{eq:auc-obj}
	p(1-p)\Bigl[1- \text{AUC}(\bw)\Bigr] \le f(\bw):= p(1-p) \EX[( 1 -\bw^\top (x - x') )^2 | y=1,y'=-1 ],
\end{equation}
where $p =\text{Pr}(y=1)$.
Now the regularization framework for maximizing the AUC score can be formulated as follows
\begin{equation}\label{auc-regularizaiton}
  \min_{\bw\in\rbb^d}\;\Bigl\{\phi(\bw):=  f(\bw) +\Omega(\bw) \Bigr\},
\end{equation}
where $\Omega:\rbb^d\mapsto\rbb^+$ is a convex regularizer. This pairwise nonlinearity in the sampling distribution makes the direct deployment of standard SGD infeasible.

\subsection{Related Work} There are considerable efforts on developing optimization algorithms for AUC maximization, which can roughly be divided into three categories.

The first category is batch learning algorithms for AUC maximization with focus on the empirical risk minimization \citep{Cortes2} which use the training data at once.  For instance, the early work \citep{Joachims,herschtal2004optimising} proposed to use the cutting plane method and  gradient descent algorithm, respectively.  \cite{Xinhua} developed an appealing algorithmic framework for optimizing the multivariate performance measures \citep{Joachims} including the AUC score and precision-recall break-even point. The algorithms there used the smoothing techniques \citep{Nest:2005} and the Nesterov's accelerated gradient algorithm \citep{nesterov1983method}.
Support Vector Algorithms were proposed to maximize the partial area under the ROC curve between any two false positive rates, which is interesting in several applications, e.g., ranking, biometric screening and medicine~\citep{NA17}.
Such batch learning algorithms generally require $\O\bigl(\min\bigl({1\over \gep},{1\over \sqrt{\gl \gep}}\bigr)\bigr)$ iterations to achieve an accuracy of $\epsilon$, but have a high per-iteration cost of $\O(n d).$ Here, $\gl, n$, and $d$ are the regularization parameter, the number of samples, and the dimension of the data, respectively.
Such algorithms train the model on the whole training data which are not suitable for analyzing massive streaming data that arrives continuously.

The second category of work \citep{kar2013generalization,wang2012generalization,Ying} extended the classical online gradient descent (OGD)  \citep{zinkevich2003online,hazan2016introduction,shalev2012online} to the setting of pairwise learning and hence is applicable to the problem of AUC maximization.  Regret bounds were established there which can be converted to generalization bounds in the stochastic setting as shown by \cite{kar2013generalization,wang2012generalization}. Such algorithms, however, need to compare the latest arriving data with previous data which require to store the historic data. This leads to expensive space and per-iteration complexities  $\O(td)$ at the $t$-th iteration which is not feasible for streaming data.  For the specific least square loss, \citet{gao2013one} developed an one-pass AUC maximization method by updating the covariance matrices of the training data, which has $\O(d^2)$ space and per-iteration time complexity which could be problematic for high-dimensional data.

The third category of work \citep{YWL,liu2018fast,natole2018stochastic} considered the expected risk and used primal-dual SGD algorithms. In particular, \cite{YWL,natole2018stochastic} formulated AUC maximization \eqref{auc-regularizaiton} as a saddle point problem as follows
\begin{equation}\label{auc-saddle}
  \min_{\bw,a,b\in\rbb}\max_{\alpha\in\rbb}\;\ebb_{z}\big[F(\bw,a,b,\alpha;z)\big] + \gO(\bw),
\end{equation}
where $F(\bw,a,b,\alpha;z)=p(1-p)+(1-p)(\bw^\top x-a)^2\ibb_{[y=1]}+p(\bw^\top x-b)^2\ibb_{[y=-1]}
    +2(1+\alpha)\bw^\top x\big(p\ibb_{[y=-1]}-(1-p)\ibb_{[y=1]}\big)-p(1-p)\alpha^2.$
Then, they proposed to perform SGD on both the primal variables $\bw,a$ and $b$, and  the dual variable $\alpha.$ This algorithm has per-iteration and space cost of $\O(d)$, making them amenable for streaming data analysis.  It enjoys a moderate convergence rate $\O(1/\sqrt{T}).$  The most recent work by \citet{liu2018fast} also used this saddle point formulation  and developed a novel multi-stage scheme for running primal-dual stochastic gradient algorithms which enjoy a fast convergence of $\widetilde{\O}({1/T})$\footnote{We use the notation $\widetilde{\O}$ to hide polynomial of logarithms.} for non-strongly-convex objective functions. Both algorithms in \cite{YWL,liu2018fast} require a critical assumption of uniform boundedness for model parameters. i.e. $\|\bw\|\le R$ which might be difficult to adjust in practice. \citet{natole2018stochastic} developed a stochastic proximal algorithm for AUC maximization with a convergence rate $\widetilde{\O}(1/T)$ for strongly convex objective function. The potential limitation of this method is that it assumes the conditional expectations $\EX[x|y=1]$ and $\EX[x'| y'=-1]$ are known a priori which is hard to satisfy in practice.

There are some other related work. For instance, \cite{palaniappan2016stochastic} developed an appealing stochastic primal-dual algorithm for saddle point problems with convergence rate of $\O({1\over T})$ which, as a by-product,  can be applied to AUC maximization with the least square loss.  However, their saddle point formulation focused on the empirical risk minimization  and can not be applied to the population risk in our case. In addition, the primal-dual algorithm there requires  strong convexity on both the primal and dual variables, and the algorithm has per-iteration complexity $\O(n+d)$ where $n$ is the total number of training samples and $d$ is the dimension of the data.

Our work fall in the regime where the aim is to minimize an expected-valued objective function which is nonlinear with respect to the sampling distribution.  This research area is attracting more and more attention in optimization and machine learning with important applications to reinforcement learning and robust learning.  For example, \cite{wang2016accelerating,wang2017stochastic} proposed a stochastic compositional gradient descent (SCGD) for solving the problem
\begin{equation}\label{eq:composite} \min_{\bw\in \gO}  \EX[f_v (\EX (g_w (\bw)| \bv))],\end{equation}
where $\gO$ is a closed convex set of $\R^n$,  $f_v: \R^m \mapsto \R$ and $g_w: \R^n \mapsto \R^m$ are functions parametrized by the random variables $w$ and $v$.  However, it is not clear how to formulate the problem of AUC maximization as \eqref{eq:composite}. In addition, the SCGD algorithms proposed in \citep{wang2017stochastic,wang2016accelerating}   require that both the gradients of $f_v$ and $g_w$ are bounded which is not the case for our setting since we use the least square loss. As we show soon in the next section,  we explore the intrinsic structure of AUC maximization to show our proposed algorithms are guaranteed to converge with high probability without boundedness assumptions. Moreover, it can achieve a fast convergence rate of $\widetilde{\O}({1\over T})$ without strong convexity.

\subsection{Main Contributions}
In this paper, we propose novel SGD algorithms for AUC maximization which does not need the boundedness assumptions and can achieve a fast convergence rate without strong convexity. Our key idea is the new decomposition technique (see Proposition \ref{lem:unbiased}) which directly works with the objective function motivated by the saddle point formulation \citep{YWL,natole2018stochastic}.  From this new decomposition, we are able to design approximately unbiased estimators for the true gradient $\nabla f (\bw)$. Our algorithms do not need to store the previous data points in contrast to the approaches in \citep{wang2012generalization,kar2013generalization,zhao2011online} or accessing true conditional expectations as in \citep{natole2018stochastic}.  A comparison of our algorithm with other methods is summarized in Table \ref{tab:comparison}.

\begin{table}[t]
\small
\setlength{\tabcolsep}{4pt}
  \centering\def\arraystretch{1.2}
\centering
  \begin{tabular}{|c|c|cc|c|}\hline
  Algorithm & storage/per-iteration & bound type & rate & penalty \\ \hline
  OAM \citep{zhao2011online} & $\O(Bd)$   & regret & $\O(1/\sqrt{T})$ & $\ell_2$   \\ \hline
  OPAUC \citep{gao2013one} & $\O(d^2)$  & regret & $\O(1/\sqrt{T})$  & $\ell_2$   \\ \hline
  SOLAM \citep{YWL} & $\O(d)$   & w.h.p. & $\O(1/\sqrt{T})$ & $\ell_2$- constraint \\ \hline
  FSAUC \citep{liu2018fast} & $\O(d)$  & w.h.p. & $\widetilde{\O}(1/T)$ & $\ell_1$- constraint \\ \hline
  SPAM \citep{natole2018stochastic} & $\O(d)$   & expectation & $\widetilde{\O}(1/T)$ & strongly convex   \\ \hline
  SPAUC (this work) & $\O(d)$   & w.h.p. & $\widetilde{\O}(1/\sqrt{T})$ &   convex regularizer \\ \hline
  \multirow{2}{*}{SPAUC (this work)}   & \multirow{2}{*}{$\O(d)$} & \multirow{2}{*}{ w.h.p.}  & \multirow{2}{*}{$\widetilde{\O}(1/T)$ }  & strongly convex \\
  &&  & & or no regularizer \\ \hline
  \end{tabular}
  \normalsize
  \caption{Comparison of different AUC maximization methods. The notation $B$ refers to the buffer size in \citet{zhao2011online}. For the bound type, ``regret'' refers to regret bounds, ``expectation'' refers to convergence rates in expectation and ``w.h.p.'' refers to convergence rates with high probabilities. If the bound type is ``regret'', we use the rate $\O(1/\sqrt{T})$ to mean regret bounds $\O(\sqrt{T})$ for a consistent comparison.
		\label{tab:comparison}} 		
\end{table}

From the side of technical novelty, we develop techniques to control the norm of iterates with high probabilities, and hence there is no boundedness assumptions on the iterates. Essential components includes controlling (weighted) summation of function values by self-bounding property of loss functions (Corollary \ref{lem:boundness}), probabilistic bounds on approximating unbiased stochastic gradients with empirical counterparts (Lemma \ref{lem:approx}) and the trick of offsetting the conditional variances of some martingales by some other terms due to the intrinsic property of the objective function.  Our major contributions can be summarized as follows.

\begin{itemize}

\item We propose a novel stochastic proximal algorithm for AUC maximization which accommodates general convex regularizers with favorable $\O(d)$ space and per-iteration time complexities. Our algorithm is gradient-based and hence is simple and easy to implement which does not need the multi-stage design \citep{liu2018fast} and bounded assumption on model parameters \citep{liu2018fast,YWL}.

\item  We establish a convergence rate $\widetilde{O}(1/\sqrt{T})$ with high probability for our algorithm with $T$ iterations, and improve it to  a fast convergence $\widetilde{\O}(1/T)$ for both cases of no regularization term (non-strong convexity) and strongly convex regularizers.

\item We perform a comprehensive empirical comparison against five state-of-the-art AUC maximization algorithms over sixteen  benchmark data sets from real-world application domains. Experimental results show that our algorithm can achieve superior performance  with a consistent and significant reduction in running time.
\end{itemize}

\noindent {\bf Organization of the paper}. The remainder of this paper is organized as follows. We state the algorithm with motivation in Section \ref{sec:alg}. Theoretical and experimental results are presented in Section \ref{sec:theory} and Section \ref{sec:exp}, respectively. The proofs of theoretical results are given in Section \ref{sec:proof}. We conclude the paper in Section \ref{sec:conclusion}.

\section{Proposed Algorithm \label{sec:alg}}
 Our objective is to develop  efficient SGD-type algorithms for AUC maximization scalable to large scale streaming data.  In particular, we aim to design an (approximately) unbiased estimator for the true gradient $\nabla f(\bw)$ with per-iteration cost $\O(d)$ to perform SGD-type algorithms.   In particular, our new design is mainly motivated by the saddle point formulation in \citep{YWL,natole2018stochastic}.

 To illustrate the main idea, let
 \begin{multline}\label{tf}
  \widetilde{F}(\bw;z)=p(1-p)+ (1-p)\big(\bw^\top\big(x-\ebb[\tilde{x}|\tilde{y}=1]\big)\big)^2\ibb_{[y=1]}\\
  +p\big(\bw^\top\big(x-\ebb[\tilde{x}|\tilde{y}=-1]\big)\big)^2\ibb_{[y=-1]}
  +2p(1-p)\bw^\top\Big(\ebb[x'|y'=-1]-\ebb[x|y=1]\Big)\\
  +p(1-p)\big(\bw^\top\big(\ebb[x'|y'=-1]-\ebb[x|y=1]\big)\big)^2.
\end{multline}
It was shown in \citep{YWL,natole2018stochastic} that  the saddle point formulation \eqref{auc-saddle} implies that $f(\bw) =  \min_{\bw,a,b}\max_\alpha   \EX[F(\bw,a,b,\alpha; z)].$ In particular, for any fixed $\bw$ the optima $a,b,\alpha$ have a closed-form solution of $a(\bw),b(\bw)$ and $\alpha(\bw)$ which are given by
\begin{equation}\label{closed-form}
a(\bw)=\bw^\top\ebb[x|y=1],\; b(\bw)=\bw^\top\ebb[x'|y'=-1],\; \alpha(\bw)=b(\bw) - a(\bw).
\end{equation}
Indeed,   let $F_1(\bw;z)=F(\bw, a(\bw),b(\bw), \ga(\bw); z)$ and then
\begin{multline*}
  F_1(\bw;z)=(1-p)\big(\bw^\top\big(x-\ebb[\tilde{x}|\tilde{y}=1]\big)\big)^2\ibb_{[y=1]}
  +p\big(\bw^\top\big(x-\ebb[\tilde{x}|\tilde{y}=-1]\big)\big)^2\ibb_{[y=-1]}\\
  +2\big(1+\bw^\top\big(\ebb[x'|y'=-1]-\ebb[x|y=1]\big)\big)\bw^\top x\big(p\ibb_{[y=-1]}-(1-p)\ibb_{[y=1]}\big)\\
  +p(1-p)-p(1-p)\big(\bw^\top\big(\ebb[x'|y'=-1]-\ebb[x|y=1]\big)\big)^2.
\end{multline*}
Note that $\bw^\top \ebb\big[x\big(p\ibb_{[y=-1]}-(1-p)\ibb_{[y=1]}\big)\big]=p(1-p)\bw^\top\big(\ebb[x'|y'=-1]-\ebb[x|y=1]\big).$ After organizing the terms,  one can easily see that $\ebb[\widetilde{F}(\bw;z)]=\ebb[F_1(\bw;z)] =f(\bw)$ for any $\bw$.   Consequently, one can see that both $\nabla \widetilde{F}(\bw;z)$  and $\nabla {F}_1(\bw,z)$ are both unbiased estimators of $\nabla f(\bw)$, i.e. $\EX[\nabla F_1(\bw;z)]= \EX[\nabla \widetilde{F}(\bw;z) ] = \nabla f(\bw).$ The work of \citet{natole2018stochastic} proposed to use $\nabla F(\bw,a(\bw),b(\bw),\alpha(\bw);z)$ as an unbiased gradient and the convergence analysis was proved in expectation.   It is easy to see that $F_1(\bw; z)$ is not convex, i.e.  the Hessian of $F_1(\bw;z)$ is not positive-semi-definite (PSD). The non-convexity of $F_1(\bw; z)$ presents daunting difficulties to bound the iterates and  deriving the convergence of the algorithm in high-probability using concentration inequalities.  In contrast, the new design of $\widetilde{F}(\bw;z)$ is convex with respect to $\bw$ which will enable us to prove convergence in high probability.

In a nutshell, we have the following important proposition. Motivated by the saddle-point formulation in \citep{YWL,natole2018stochastic} as mentioned above,  we also give an alternative but self-contained proof by writing the objective function as
\begin{align*}\label{eq:decomp0}(1 - \bw^\top (x - x') )^2  & =
\bigl( [1 +  \alpha(\bw) ] + [\bw^\top x' - b(\bw)] - [\bw^\top x - a(\bw)]\bigr)^2 \\
& = \bigl( \bigl[1 +  \bw^\top (\EX[x'| y'=-1] - \EX[x|y=1])\bigr] \\
& + \bigl[\bw^\top (x' - \EX [x'| y'=-1])  - \bw^\top (x - \EX[x| y=1])\bigr]\bigr)^2\numberthis. \end{align*}

\begin{proposition}\label{lem:unbiased}
  For any $\bw$, we have
  \begin{equation}\label{unbiased}
    \ebb\big[\widetilde{F}(\bw;z)\big]=f(\bw)\quad\text{and}\quad \ebb\big[\widetilde{F}'(\bw;z)\big]=\nabla f(\bw),
  \end{equation}
  where we use the abbreviation $\widetilde{F}'(\bw;z):=\frac{\partial \widetilde{F}(\bw;z)}{\partial \bw}$.
  Furthermore, for any $z$ the function $\widetilde{F}(\bw;z)$ is a convex function of $\bw$.
\end{proposition}
\begin{proof} As indicated by \eqref{eq:decomp0}, we
 write $(1 - \bw^\top (x - x') )^2 $ as three terms:
\begin{align*}   &  \Bigl( \bigl[1 +  \bw^\top (\EX[x'| y'=-1] - \EX[x|y=1])\bigr] + \bigl[\bw^\top (x' - \EX [x'| y'=-1])  - \bw^\top (x - \EX[x| y=1])\bigr]\Bigr)^2 \\
& \hspace*{-2mm}= \bigl\{\bigl[1 +  \bw^\top (\EX[x'| y'=-1] - \EX[x|y=1])\bigr]^2\bigr\} \hspace*{-1mm}+ \hspace*{-1mm}\bigl\{\bigl[\bw^\top (x' - \EX [x'| y'=-1])  - \bw^\top (x - \EX[x| y=1])\bigr]^2 \bigr\} \\
& +  \bigl\{2 \bigl[1 +  \bw^\top (\EX[x'| y'=-1] - \EX[x|y=1])\bigr] \bigl[\bw^\top (x' - \EX [x'| y'=-1])  - \bw^\top (x - \EX[x| y=-1])\bigr]\bigr\} \\
& =  \bI + \mathbf{II} + \mathbf{III}.
 \end{align*}
It suffices to estimate the above terms one by one.  To this end, the first term is deterministic, and hence
\begin{align*}\label{eq:termI}
	\EX[\, \bI \,  |  y=1, y'=-1] & =  2\bw^\top\bigr(\ebb[x'|y'=-1]-\ebb[x|y=1]\bigl)
  \\ & +\big(\bw^\top\big(\ebb[x'|y'=-1]-\ebb[x|y=1]\big)\big)^2+ 1. \numberthis
\end{align*}
For the second term, noticing that $\EX \bigl[ \bw^\top (x' - \EX [x'| y'=-1]) \bw^\top (x - \EX[x| y=1]) \, | y=1,y'=-1\bigr]  = \EX \bigl[ \bw^\top (x' - \EX [x'| y'=-1]) \, | y'=-1\bigr] \EX \bigl[ \bw^\top (x - \EX [x| y=1]) \, | y=1\bigr] =0  $ as $(x,y)$ and $(x',y')$ are independent, we have
\begin{align*}\label{eq:termII}  & \EX[\,  \mathbf{II}\, | y=1, y'=-1]  \\ & =
\EX\bigl[ \bigl(\bw^\top (x' - \EX [x'| y'=-1])  - \bw^\top (x - \EX[x| y=1])\bigr)^2 |  y=1, y'=-1\bigr] \\
& = \EX\bigl[ \bigl(\bw^\top (x' - \EX [x'| y'=-1])\bigr)^2 |   y=1, y'=-1\bigr] + \EX\bigl[ \bigl(\bw^\top (x - \EX [x| y=1])\bigr)^2 |   y=1, y'=-1\bigr] \\
& = \EX\bigl[ \bigl(\bw^\top (x' - \EX [x'| y'=-1])\bigr)^2 |  y'=-1\bigr] + \EX\bigl[ \bigl(\bw^\top (x - \EX [x| y=1])\bigr)^2 |  y=1\bigr] \\
& = {1\over 1-p}\EX\bigl[ \bigl(\bw^\top (x' - \EX [x'| y'=-1])\bigr)^2 \ibb_{[y'=-1]} \bigr] + {1\over p }\EX\bigl[ \bigl(\bw^\top (x - \EX [x| y=1])\bigr)^2  \ibb_{[y=1]}\bigr].
\numberthis
\end{align*}
For the third term,
\begin{align*}\label{eq:termIII}
& \EX[\, \mathbf{III}\,  |  y =1, y'=-1] = 2\bigl[1 +  \bw^\top (\EX[x'| y'=-1] - \EX[x|y=1])\bigr] \\ & \times \EX\bigl[ \bw^\top (x' - \EX [x'| y'=-1])  - \bw^\top (x - \EX[x| y=-1])  |   y =1, y'=-1\bigr] =0.
\numberthis 	
\end{align*}
Combining equations \eqref{eq:termI},\eqref{eq:termII}, and \eqref{eq:termIII}together, we have
\begin{align*} f(\bw) & = p(1-p) \EX[(1 - \bw^\top (x - x') )^2 | y=1, y'=-1 ]  \\ & = 2p(1-p)\bw^\top\bigr(\ebb[x'|y'=-1]-\ebb[x|y=1]\bigl)
  \\ & +p(1-p)\big(\bw^\top\big(\ebb[x'|y'=-1]-\ebb[x|y=1]\big)\big)^2+ p(1-p)  \\
  &  + {p}\EX\Bigl[ \bigl(\bw^\top (x' - \EX [x'| y'=-1])\bigr)^2 \ibb_{[y'=-1]} \Bigr] + {(1-p) }\EX\Bigl[ \bigl(\bw^\top (x - \EX [x| y=1])\bigr)^2  \ibb_{[y=1]}\Bigr], \end{align*}
which implies that $f(\bw) = \EX [\widetilde{F}(\bw;z)].$

The fact of $\ebb\big[\widetilde{F}'(\bw;z)\big]=\nabla f(\bw)$ follows directly from the Leibniz's integral rule that  the derivative and the integral can be interchangeable as $F$ is a quadratic function and  the input $x$ is from a bounded domain.

For the last statement, notice that
\begin{multline*}
\nabla^2\widetilde{F}(\bw;z)= 2(1-p)\big(x-\ebb[\tilde{x}|\tilde{y}=1]\big)\big(x-\ebb[\tilde{x}|\tilde{y}=1]\big)^\top\ibb_{[y=1]}\\
+2p(x-\ebb[\tilde{x}|\tilde{y}=-1])(x-\ebb[\tilde{x}|\tilde{y}=-1])^\top\ibb_{[y=-1]}\\
  +2p(1-p)\big(\ebb[x'|y'=-1]-\ebb[x|y=1]\big)\big(\ebb[x'|y'=-1]-\ebb[x|y=1]\big)^\top.
\end{multline*}
It is clear that $\nabla^2\widetilde{F}(\bw;z)$ is positive semi-definite, and hence $\widetilde{F}(\bw;z)$ is a convex function of $\bw$ for any $z$.  This completes the proof of the proposition.
\end{proof}

Proposition \ref{lem:unbiased} indicates to use $\widetilde{F}'(\bw;z)$ as an unbiased estimator for the gradient $\nabla f(\bw)$. However, the function $\widetilde{F}$ requires the unknown information $p,\ebb[x|y=1]$ and $\ebb[x|y=-1]$, which is unknown in practice.  We propose to replace them by
their empirical counterpart defined as follows
\begin{equation}\label{approx}
p_t=\frac{\sum_{i=0}^{t-1}\ibb_{[y_i=1]}}{t},\quad
u_t=\frac{\sum_{i=0}^{t-1}x_i\ibb_{[y_i=1]}}{\sum_{i=0}^{t-1}\ibb_{[y_i=1]}},\quad
v_t=\frac{\sum_{i=0}^{t-1}x_i\ibb_{[y_i=-1]}}{\sum_{i=0}^{t-1}\ibb_{[y_i=-1]}},
\end{equation}
where we reserve an example $(x_0,y_0)$ drawn independently from $\rho$.
The resulting estimator for $F$ at time $t$  then becomes
\begin{multline*}
  \hat{F}_t(\bw;z)=(1-p_t)\big(\bw^\top\big(x-u_t\big)\big)^2\ibb_{[y=1]} +p_t\big(\bw^\top\big(x-v_t\big)\big)^2\ibb_{[y=-1]}\\
  +2p_t(1-p_t)\bw^\top\big(v_t-u_t\big)+p_t(1-p_t)\big(\bw^\top\big(v_t-u_t\big)\big)^2+p_t(1-p_t).
\end{multline*}
It is easy to verify by computing its Hessian that  $\hat{F}_t(\bw;z)$ is convex with respect to $\bw.$ Its gradient can be directly computed as follows
\begin{multline}\label{grad-hf}
  \hat{F}_t'(\bw;z):=\frac{\partial \hat{F}_t(\bw;z)}{\partial \bw}= 2(1-p_t)\big(x-u_t\big)\big(x-u_t\big)^\top\bw\ibb_{[y=1]}
  +2p_t\big(x-v_t\big)\big(x-v_t\big)^\top\bw\ibb_{[y=-1]} \\ +2p_t(1-p_t)\big(v_t-u_t\big)+2p_t(1-p_t)\big(v_t-u_t\big)\big(v_t-u_t\big)^\top\bw.
\end{multline}
Note the stochastic gradient $\hat{F}_t'(\bw;z)$ can be efficiently computed with an arithmetic cost $\O(d)$ and we do not need to store covariance matrices.

\medskip
\noindent{\bf Algorithm:} We propose to solve this regularization problem \eqref{auc-regularizaiton} by the following {\em Stochastic proximal AUC maximization} (SPAUC) algorithm with $\bw_1=0$ and for any $t\ge 1$,
\begin{equation}\label{SAUC}
  \bw_{t+1}=\arg\min_{\bw\in\rbb^d}\eta_t\big\langle \bw-\bw_t,\hat{F}'_t(\bw_t;z_t)\big\rangle+\eta_t\Omega(\bw)+\frac{1}{2}\|\bw-\bw_t\|_2^2,
\end{equation}
where $\{\eta_t\}_t$ is a sequence of positive step sizes and $z_t$ is drawn independently from $\rho$ at the $t$-th iteration. At the $t$-th iteration, SPAUC
builds a temporary objective function consisting of three components:
a first order approximation of $f(\bw)$ based on the stochastic gradient $\hat{F}_t'(\bw;z)$, a regularizer kept intact to preserve
a composite structure and a term $\frac{1}{2}\|\bw-\bw_t\|_2^2$ to make sure the upcoming iterate $\bw_{t+1}$ not far away from the current iterate.

We comment that this stochastic proximal algorithm has been developed in standard classification and regression \citep{duchi2009efficient,rosasco2014convergence}. The convergence results in expectation were proved under the boundedness assumptions either on the stochastic estimators or on the iterates \citep{Duchi} and under strong convexity condition \citep{rosasco2014convergence}.  Our theoretical objective is to establish convergence analysis  with high probability without boundedness assumptions and fast convergence rates even without strong convexity.
Furthermore, we need to handle the bias of $\hat{F}_t'(\bw;z)$ as an estimator of the gradient due to the approximation strategy \eqref{approx}.

\section{Main Convergence Results\label{sec:theory}}
In this section, we present theoretical  convergence rates with high probability  for SPAUC.  We consider two types of objective functions of the form \eqref{auc-regularizaiton}: AUC maximization with a convex $\phi$ and AUC maximization with $\phi$ satisfying a quadratic functional growth.
Let $S^*=\{\bw\in\rbb^d:\phi(\bw)=\min_{\tilde{\bw}}\phi(\tilde{\bw})\}$ be the set of minimizers.
For any $\bw$, we denote by $\bw^*=\arg\min_{\tilde{\bw}\in S^*}\|\bw-\tilde{\bw}\|_2$ the projection of $\bw$ on to $S^*$,
where for any $p\geq1$ and $\bw=(w_1,\ldots,w_d)^\top$, we denote $\|\bw\|_p=\big[\sum_{j=1}^{d}|w_j|^p\big]^{\frac{1}{p}}$.
Throughout the paper, we assume $\|\bw_1^*\|_2<\infty$.

\subsection{General Convergence Rates}

In this subsection, we present convergence rates for the general regularization framework for AUC maximization. To this aim, we need to impose a so-called self-bounding property on the regularizers, meaning the subgradients can be bounded in terms of function values. We denote by $\Omega'(\bw)$ a subgradient of $\Omega$ at $\bw$ and assume $\Omega(0)$=0.
\begin{assumption}[Self-bounding property]\label{ass:self-bounding}
There exist constants $A_1,A_2\geq0$ such that the convex regularizer $\Omega$ satisfies
\begin{equation}\label{self-bounding}
  \quad\|\Omega'(\bw)\|^2_2\leq A_1\Omega(\bw)+A_2,\quad\text{for all }\bw\in\rbb^d.
\end{equation}
\end{assumption}
This self-bounding assumption above is very mild as many regularizers satisfy self-bounding property, including all
smooth regularizers and  all Lipschitz regularizers.
For example, if $\Omega(\bw)=\lambda\|\bw\|_2^2$, then \eqref{self-bounding} holds with $A_1=4\lambda$ and $A_2=0$.
If $\Omega(\bw)=\lambda\|\bw\|_1$, then \eqref{self-bounding} holds with $A_1=0$ and $A_2=\lambda^2$.
It is reasonable to assume a small regularization parameters in practice (e.g., $\lambda\leq1$), in which case we can take universal constants $A_1$ and $A_2$ for the above mentioned regularizers.

Our theoretical analysis requires to estimate $\|\bw_t\|_2$, which is achieved by the following lemma to be proved in Section \ref{sec:boundedness}. Essentially, it shows that $\|\bw_t\|_2$ is bounded (ignoring logarithmic factors)
if we consider step sizes satisfying \eqref{bound-prob-cond}. This result shows that the complexity of $\bw_t$ is well controlled even if the
iterates are updated in an unbounded domain. Let $\kappa:=\max\{1,\sup_{x\in\xcal}\|x\|_2\}$ and $C_1=\max\{A_1,16\kappa^2\}$.

\begin{theorem}\label{thm:bound-prob}
  Let $\{\bw_t\}_t$ be produced by \eqref{SAUC} with $\eta_t\leq (2C_1)^{-1}$ and $\eta_{t+1}\leq\eta_t$. We suppose Assumptions \ref{ass:self-bounding} holds,
  \begin{equation}\label{bound-prob-cond}
    \sum_{t=1}^{\infty}\eta_t\sqrt{\log t}/\sqrt{t}<\infty\quad\text{and}\quad
    \sum_{t=1}^{\infty}\eta_t^2<\infty.
  \end{equation}
  Then for any $\delta\in(0,1)$, there exists a constant $C_2$ independent of $T$ (explicitly given in the proof) such that the following inequality holds with probability at least $1-\delta$
  \begin{equation}\label{bound-prob}
    \max_{1\leq\tilde{t}\leq T}\|\bw_{\tilde{t}}-\bw_1^*\|_2^2\leq C_2\log (2T/\delta).
  \end{equation}
\end{theorem}

We are now ready to present convergence rates for SPAUC applied to general AUC objectives. In Theorem \ref{thm:rate} we present
general convergence rates in terms of step sizes satisfying \eqref{bound-prob-cond}, which are then instantiated in Corollary \ref{cor:ave-g} by
specifying step sizes. 
The convergence rate $\O(T^{-\frac{1}{2}}\log^{\frac{3+\beta}{2}}\frac{T}{\delta})$ is optimal up to a logarithmic factor for stochastic algorithms applied to general convex optimization problems~\citep{agarwal2009information}.
\begin{theorem}\label{thm:rate}
  Let the conditions of Theorem \ref{thm:bound-prob} hold. Then, for any $\delta\in(0,1)$ there exists a constant $C_3$ independent of $T$ such that the following inequality holds with probability at least $1-\delta$
  $$
  \phi(\bar{\bw}_T^{(1)})-\inf_{\bw}\phi(\bw)\leq C_3\big(\sum_{t=1}^{T}\eta_t\big)^{-1}\max\Big\{\sum_{t=1}^{T}\eta_t^2,\sum_{t=1}^{T}\eta_t/\sqrt{t}\Big\}\log^{\frac{3}{2}}\frac{T}{\delta},
  $$
  where $\bar{\bw}_T^{(1)}=\sum_{t=1}^{T}\eta_t\bw_t/\sum_{t=1}^{T}\eta_t$ is a weighted average of the first $T$ iterates.
\end{theorem}
\begin{corollary}\label{cor:ave-g}
  Let $\{\bw_t\}_t$ be produced by \eqref{SAUC} and $\delta\in(0,1)$. Suppose Assumptions \ref{ass:self-bounding} holds and $\eta_1\leq (2C_1)^{-1}$.
  \begin{enumerate}[(1)]
    \item If we choose $\eta_t=\eta_1t^{-\theta}$ with $\theta>1/2$, then with probability $1-\delta$ we have $\phi(\bar{\bw}_T^{(1)})-\inf_{\bw}\phi(\bw)=O\big(T^{\theta-1}\log^{\frac{3}{2}}\frac{T}{\delta}\big)$;
    \item If we choose $\eta_t=\eta_1(t\log^\beta (et))^{-\frac{1}{2}}$ with $\beta>2$,
  then with probability $1-\delta$ we have $\phi(\bar{\bw}_T^{(1)})-\inf_{\bw}\phi(\bw)=O\big(T^{-\frac{1}{2}}\log^{\frac{3+\beta}{2}}\frac{T}{\delta}\big)$.
  \end{enumerate}
\end{corollary}
The proofs for Theorem \ref{thm:rate} and Corollary \ref{cor:ave-g} can be found in Section \ref{sec:proof-rate}.

\subsection{Fast Convergence Rates}
In this subsection, we show that a faster convergence rate is possible for SPAUC if a quadratic functional growth condition is imposed to the objective function~\citep{anitescu2000degenerate,necoara2018linear}.
\begin{assumption}[Quadratic functional growth]\label{ass:strong}
  We assume the existence of $\sigma_\phi>0$ such that
  \begin{equation}\label{error-bound}
  \phi(\bw)-\phi(\bw^*)\geq \sigma_\phi\|\bw-\bw^*\|_2^2,\quad\text{for all }\bw\in\rbb^d.
  \end{equation}
\end{assumption}

The quadratic functional growth assumption \eqref{error-bound} means that the objective function grows faster than the
squared distance between any feasible point and the optimal set~\citep{necoara2018linear}.
This condition is milder than assuming a strong convexity of $\phi$~\citep{necoara2018linear}. Indeed, it holds if $\Omega(\bw)=\lambda\|\bw\|_2^2$. It also holds if we consider no regularization, i.e., $\Omega(\bw)=0$ as shown in the next proposition.
We give the proof for completeness.
\begin{proposition}\label{prop:quadratic}
  The function $\phi(\bw)=f(\bw)$ satisfies Assumption \ref{ass:strong}.
\end{proposition}
\begin{proof}   Indeed, the objective function can be written as
  \[
    f(\bw)= p(1-p)\big(\|A\bw\|_2^2+\bw^\top c +1\big)
  \]
  with $A\in\rbb^{d\times d}$ being a symmetric matrix  satisfying $A^2=\ebb[(x-x')(x-x')^\top|y=1,y'=-1]$ and $c=-2\ebb[x-x'|y=1,y'=-1]$.
  Analyzing analogously to the proof of Theorem 9 in \citet{necoara2018linear}, one can show that $S^*=\big\{\bw: A\bw=g^*\big\}$ for some $g^*\in\rbb^d$. By the definition of $\bw^*$ we know that $\bw-\bw^*$ is orthogonal to the kernel of $A^2$ and therefore
  $\lambda_{\text{min}}(A^2)\|\bw-\bw^*\|_2^2\leq \|A\bw-A\bw^*\|_2^2$, where $\lambda_{\text{min}}(A^2)$ denotes the smallest non-zero eigenvalue of $A^2$. Furthermore, we know
  \begin{align*}
     & p^{-1}(1-p)^{-1}\big(f(\bw)-f(\bw^*)\big)  =\|A\bw\|_2^2-\|A\bw^*\|_2^2+(\bw-\bw^*)^\top c \\
     & =\|A\bw-A\bw^*\|_2^2+2\big(A\bw-A\bw^*\big)^\top A\bw^*+(\bw-\bw^*)^\top c = \|A\bw-A\bw^*\|_2^2,
  \end{align*}
  where the last identity is due to the optimality condition $2A^\top A\bw^*+c=0$. It then follows that
  $f(\bw)-f(\bw^*)\geq p(1-p)\lambda_{\text{min}}(A^2)\|\bw-\bw^*\|_2^2$. The proof is complete.
\end{proof}

Under Assumption \ref{ass:strong}, we show with high probabilities that the suboptimality measured by both the parameter distance and function values decay with
the rate $\widetilde{\O}(T^{-1})$, which is optimal up to a logarithmic  factor~\citep{agarwal2009information}.
Let $\sigma_\Omega\geq0$ be a constant satisfying
\[
\Omega(\bw)-\Omega(\tilde{\bw})-\langle \bw-\tilde{\bw},\Omega'(\tilde{\bw})\rangle\geq 2^{-1}\sigma_\Omega\|\bw-\tilde{\bw}\|_2^2,\quad\forall \bw,\tilde{\bw}\in\rbb^d.
\]
Note $\sigma_\Omega$ can be zero and therefore our results apply to non-strongly-convex regularizers, e.g., $\Omega(\bw)=0$ for all $\bw$.
Without loss of generality, we assume $\sigma_f:=\sigma_\phi-\sigma_\Omega\geq0$. 

\begin{theorem}\label{thm:strong}
  Let $\delta\in(0,1)$. Suppose Assumption \ref{ass:self-bounding} and Assumption \ref{ass:strong} hold. Let $\{\bw_t\}_t$ be the sequence produced by \eqref{SAUC} with $\eta_t=\frac{2}{\sigma_\phi t+2\sigma_f+\sigma_\phi t_1}$,
  where $t_1\geq 32C_1\sigma_\phi^{-1}\log\frac{2T}{\delta}$. Then,  the following inequality holds with $1-\delta$
  for $t=1,\ldots,T$ ($T>2$)
  \begin{equation}\label{strong}
    \|\bw_t-\bw_t^*\|_2^2=\widetilde{\O}(1/t)\quad\text{and}\quad\phi(\bar{\bw}_t^{(2)})-\inf_{\bw}\phi(\bw)=\widetilde{\O}(1/t),
  \end{equation}
  where $\bar{\bw}_t^{(2)}$ is a weighted average of iterates defined by
  \[
    \bar{\bw}_t^{(2)} = \big(\sum_{k=1}^{t}(k+t_1+1)\big)^{-1}\sum_{k=1}^{t}(k+t_1+1)\bw_k.
  \]
\end{theorem}
The proof of Theorem \ref{thm:strong} is postponed to Section \ref{sec:proof-rate-fast}.

The following two corollaries follow directly from Theorem \ref{thm:strong} by noting the quadratic functional growth property of the associated objective functions and the self-bounding property of the regularizers. We omit the proof here for brevity.

\begin{corollary}
  Let $\delta\in(0,1)$. Let $\{\bw_t\}_t$ be the sequence produced by \eqref{SAUC} with $\eta_t=\frac{2}{\sigma_\phi t+2\sigma_f+\sigma_\phi t_1}$ and $\Omega(\bw)=\lambda\|\bw\|_2^2/2$,
  where $t_1\geq 32C_1\sigma_\phi^{-1}\log\frac{2T}{\delta}$. Then, \eqref{strong} holds w.h.p..
\end{corollary}

\begin{corollary}\label{cor:unreg}
  Let $\delta\in(0,1)$. Let $\{\bw_t\}_t$ be the sequence produced by \eqref{SAUC} with $\eta_t=\frac{2}{\sigma_\phi t+2\sigma_f+\sigma_\phi t_1}$ and $\Omega(\bw)=0$,
  where $t_1\geq 32C_1\sigma_\phi^{-1}\log\frac{2T}{\delta}$. Then, \eqref{strong} holds w.h.p..
\end{corollary}

\section{Experiments\label{sec:exp}}
In this section, we present experimental results to show the effectiveness of the proposed algorithm in achieving a satisfactory AUC with a fast convergence speed.
We first describe the baseline methods used in our experimental comparison as well as the associated parameter setting in Section \ref{sec:baseline}.
Datasets used in the experiments and detailed experimental results are presented in Section \ref{sec:dataset} and Section \ref{sec:result}, respectively.

\subsection{Baseline Methods\label{sec:baseline}}
We compare SPAUC to several state-of-the-art online AUC maximization algorithms. The algorithms we consider include
\begin{itemize}
  \item the stochastic proximal AUC maximization (SPAUC) \eqref{SAUC} with either no regularizers $\Omega(\bw)=0$ or an $\ell_2$ regularizer $\Omega(\bw)=\lambda\|\bw\|_2^2$;
  \item the stochastic proximal AUC maximization (SPAM) \citep{natole2018stochastic} with $\Omega(\bw)=\lambda\|\bw\|_2^2$;
  \item the stochastic online AUC maximization (SOLAM) \citep{YWL} based on a saddle problem formulation;
  \item the one-pase AUC maximization (OPAUC) \citep{gao2013one} which uses the first and second-order statistics of training data to compute gradients;
  \item the online AUC maximization based on the hinge loss function (OAM\_gra) \citep{zhao2011online};
  \item the fast stochastic AUC maximization (FSAUC) \citep{liu2018fast} which applies a multi-stage stochastic optimization technique to a saddle problem formulation.
\end{itemize}

\begin{table}[htbp]

\small
\setlength{\tabcolsep}{1pt}
  \centering\def\arraystretch{1.2}
\centering
  \begin{tabular}{|c|cc|c|cc|c|cc|c|cc|}\hline
  datasets & \# inst & \# feat & datasets & \# inst & \# feat & datasets & \# inst & \# feat & datasets & \# inst & \# feat \\ \hline
  diabetes & 768 & 8 & ijcnn1 & 141691 & 22 & german & 1000 & 24 & satimage & 6435 & 36 \\
  acoustic & 78823 & 50 & covtype & 581012 & 54 & a9a & 32561 & 123 & connect & 67557 & 126 \\
  usps & 9298 & 256 & w8a & 49749 & 300 &  mnist & 60000 & 780 & gisette & 7000 & 5000 \\
  real-sim & 72309 & 20958 & protein\_h & 145751 & 74 & malware & 71709 & 122  & webspam\_u & 350000 & 254  \\ \hline
  \end{tabular}
  \normalsize
  \caption{Description of the datasets used in the experiments.
		\label{tab:data_set}} 		
\end{table}

The performance of these algorithms depends on some parameters, which, as described below, we tune with the five-fold cross-validation.
For SPAUC, SPAM and SOLAM, we consider step sizes of the form $\eta_t=2/(\mu t+1)$ and validate the parameter $\mu$ over the interval
$10^{\{-7,-6.5,...,-2.5\}}$. Both SPAM and SPAUC with the $\ell_2$ regularizer requires another regularization parameter to tune, which is validated over the interval $10^{\{-5,-4,\ldots,0\}}$.
SOLAM involves the constraint on $\bw$, i.e. $\bw$ belonging to $\ell_2$-ball with radius $R$ in $\rbb^d$, for which we tune over the interval $10^{\{-1,0,\ldots,5\}}$.
For OAM\_gra, we need to tune a parameter to weight the comparison between released examples and bulk, which is validated over the interval $10^{\{-3,-2.5,\ldots,1.5\}}$.
As recommended in \citet{zhao2011online}, we fix the buffer size to $100$.
For OPAUC, we need to tune both the constant step size and the regularization parameter $\lambda$, which are validated over the interval $10^{\{-3.5,-3,\ldots,1\}}$ and $10^{\{-5,-4,\ldots,0\}}$, respectively.
The multi-stage scheme in FSAUC specifies how the step size decreases along the implementation of the algorithm and leave the initial step size as a free parameter to tune, which we validate over
the interval $10^{\{-2.5,-2,\ldots,2\}}$.
Furthermore, each iteration of FSAUC requires a projection onto an $\ell_1$-ball of radius of $R$, which we tune over the interval $10^{\{-1,0,\ldots,5\}}$.
It should be noticed that both SPAUC with no regularizers and OAM\_gra only have a single parameter to tune, while all other algorithms have two parameters to tune.
To speed up the training process, if the algorithm has two parameters $p_1,p_2$ to tune, we first construct all the possible pairs $(p_1,p_2)$ by enumerating all possible candidate values of $p_1$ and $p_2$, out of which
we randomly sample $15$ pairs without replacement to tune. 
We repeat the experiments 20 times and report the average of experimental results.

\subsection{Datasets\label{sec:dataset}}
We perform our experiments on several real-world datasets. We consider two types of datasets: the UCI benchmark dataset and the dataset in the domain of anomaly detection. The task of anomaly detection is to identify rare items, events or observations which raise suspicions by differing significantly from the majority of the data. As such,  this is suitable to test the performance of AUC maximization methods since the class there is intrinsically and highly imbalanced. We consider three datasets in the domain of anomaly detection: protein\_h, webspam\_u and malware. In particular,
webspam\_u is a subset used in the Pascal Large Scale Learning Challenge \citep{wang2012evolutionary} to
detect malicious web pages, protein\_h  is a dataset in bioinformatics used to predict which proteins are homologous to a query sentence (non-homologous sequences are labeled as anomalies)~\citep{caruana2004kdd},
and malware was collected in the Android Malware Genome Project used to detect mobile malware app~\citep{jiang2012dissecting}.
The remaining UCI datasets can be downloaded from the LIBSVM webpage~\citep{chang2011libsvm}.
For each dataset, we use 80\% of data for training and the remaining 20\% for testing.
We transform datasets with multiple class labels into datasets with binary class labels
by grouping the first half of class labels into positive labels, and grouping the remaining class labels into negative labels.
We run each algorithm until $15$ passes of the training data is reached, and report the AUC values on the test dataset.
The information of the dataset is summarized in Table \ref{tab:data_set} where we list the  UCI datasets according to the dimensionality while datasets for anomaly detection are listed at the end.

\begin{figure}[htb]
  \centering
  \hspace*{-1.2cm}
  \subfigure[diabetes]{\includegraphics[width=0.306\textwidth, trim=0 2 0 2, clip]{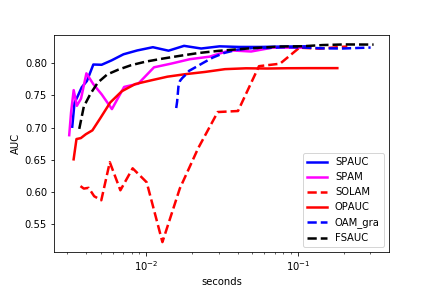}}\hspace*{-0.342cm}
  \subfigure[ijcnn1]{\includegraphics[width=0.306\textwidth, trim=0 2 0 2, clip]{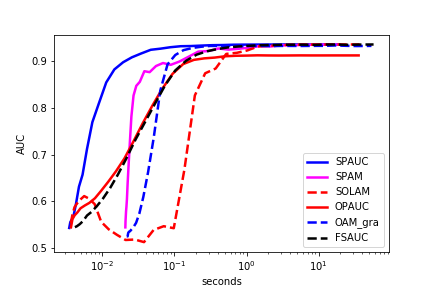}}\hspace*{-0.342cm}
  \subfigure[german]{\includegraphics[width=0.306\textwidth, trim=0 2 0 2, clip]{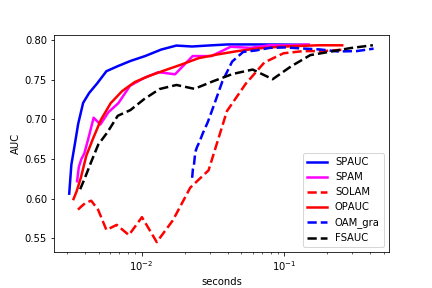}}\hspace*{-0.342cm}
  \subfigure[satimage]{\includegraphics[width=0.306\textwidth, trim=0 2 0 2, clip]{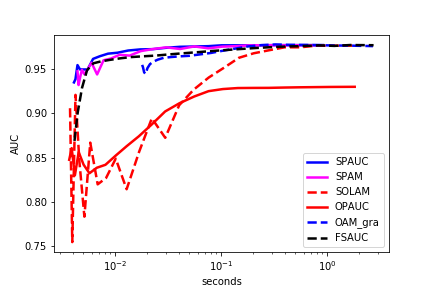}}\hspace*{-0.342cm}\\

  \hspace*{-1.2cm}
  \subfigure[acoustic]{\includegraphics[width=0.306\textwidth, trim=0 2 0 2, clip]{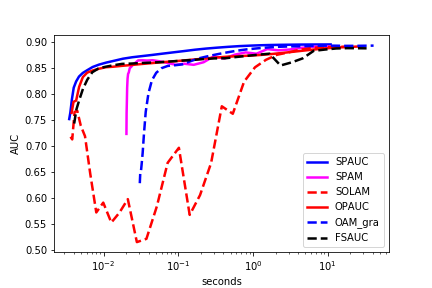}}\hspace*{-0.342cm}
  \subfigure[covtype]{\includegraphics[width=0.306\textwidth, trim=0 2 0 2, clip]{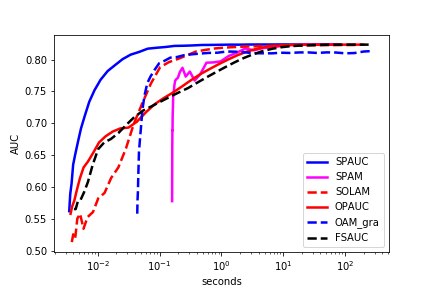}}\hspace*{-0.342cm}
  \subfigure[a9a]{\includegraphics[width=0.306\textwidth, trim=0 2 0 2, clip]{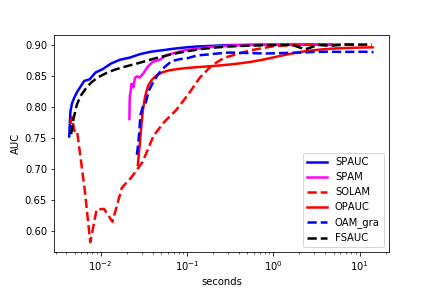}}\hspace*{-0.342cm}
  \subfigure[connect]{\includegraphics[width=0.306\textwidth, trim=0 2 0 2, clip]{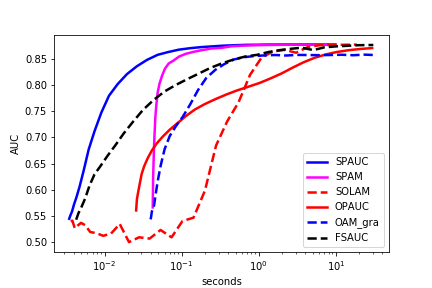}}\hspace*{-0.342cm}\\

  \hspace*{-1.2cm}
  \subfigure[usps]{\includegraphics[width=0.306\textwidth, trim=0 2 0 2, clip]{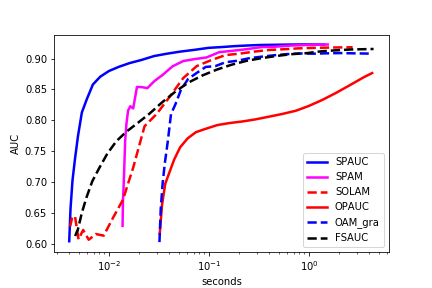}}\hspace*{-0.342cm}
  \subfigure[w8a]{\includegraphics[width=0.306\textwidth, trim=0 2 0 2, clip]{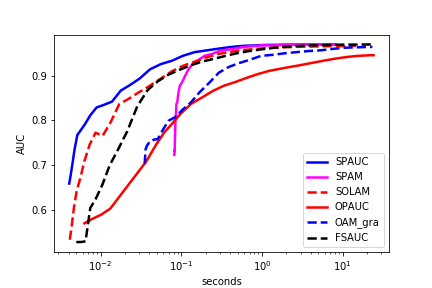}}\hspace*{-0.342cm}
  \subfigure[mnist]{\includegraphics[width=0.306\textwidth, trim=0 2 0 2, clip]{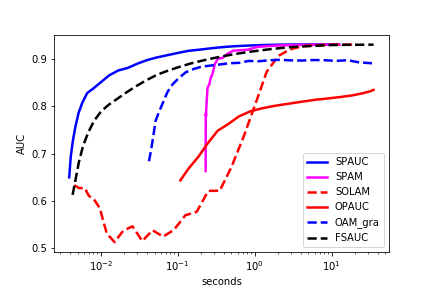}}\hspace*{-0.342cm}
  \subfigure[gisette]{\includegraphics[width=0.306\textwidth, trim=0 2 0 2, clip]{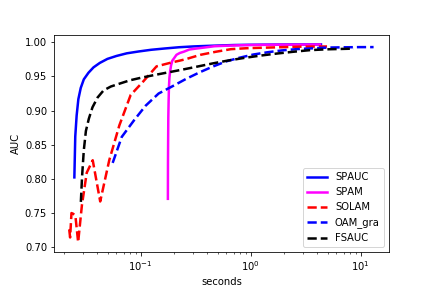}}\hspace*{-0.342cm}\\

  \hspace*{-1.2cm}
  \subfigure[real-sim]{\includegraphics[width=0.306\textwidth, trim=0 2 0 2, clip]{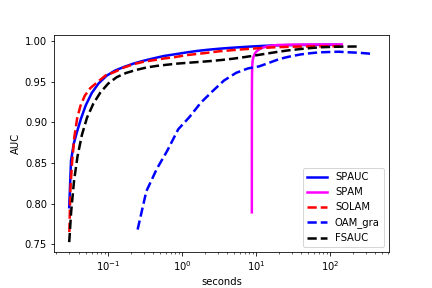}}\hspace*{-0.342cm}
  \subfigure[protein\_h]{\includegraphics[width=0.306\textwidth, trim=0 2 0 2, clip]{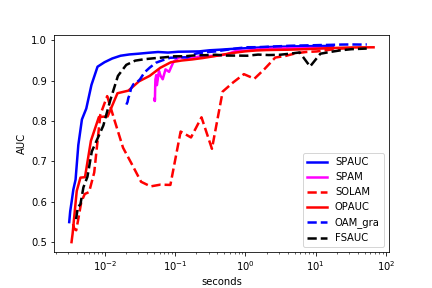}}\hspace*{-0.342cm}
  \subfigure[malware]{\includegraphics[width=0.306\textwidth, trim=0 2 0 2, clip]{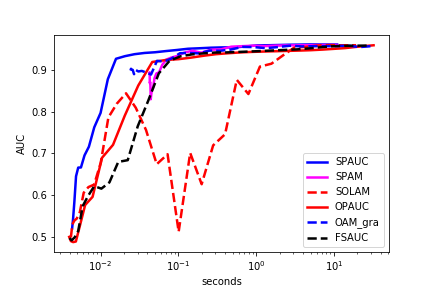}}\hspace*{-0.342cm}
  \subfigure[webspam\_u]{\includegraphics[width=0.306\textwidth, trim=0 2 0 2, clip]{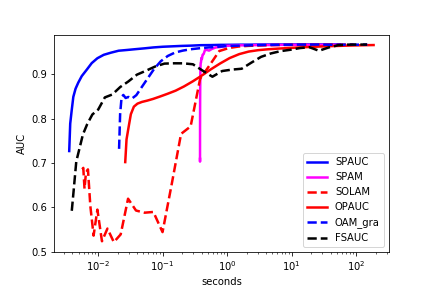}}\hspace*{-0.342cm}\\
  \vspace*{-0.6\baselineskip}
  \caption{AUC versus time curves (in seconds) for SPAUC (without regularization), SPAM, SOLAM and OPAUC, OAM\_gra and FSAUC.\label{fig:auc-all}}
\end{figure}

\begin{table}[htb]
\caption{Comparison of the testing AUC values and running time per pass (mean$\pm$std.).
		\label{tab:auc-all}}
\small
\setlength{\tabcolsep}{1.4pt}
  \centering\def\arraystretch{1.2}

\hspace*{-1cm}\begin{tabular}{|c|c|c|c|c|c|c|c|}\hline
datasets & & SPAUC & SPAM & SOLAM & OPAUC & OAM\_gra & FSAUC \\ \hline

\multirow{ 2}{*}{diabetes} & AUC & $0.8266\!\pm\! 0.0284$ & $0.8246\!\pm\! 0.0303$ & $0.8264\!\pm\! 0.0308$ & $0.7926\!\pm\! 0.0462$ & $0.8247\!\pm\! 0.0266$ & $0.8293\!\pm\! 0.0375$\\    & Time & $0.0075\!\pm\! 0.0013$ & $0.0071\!\pm\! 0.0002$ & $0.0141\!\pm\! 0.0015$ & $0.0121\!\pm\! 0.0011$ & $0.0201\!\pm\! 0.0014$ & $0.0210\!\pm\! 0.0011$\\ \hline
\multirow{ 2}{*}{ijcnn1} & AUC & $0.9361\!\pm\! 0.0019$ & $0.9358\!\pm\! 0.0018$ & $0.9362\!\pm\! 0.0019$ & $0.9127\!\pm\! 0.0021$ & $0.9331\!\pm\! 0.0031$ & $0.9361\!\pm\! 0.0015$\\    & Time & $1.2881\!\pm\! 0.0076$ & $1.3080\!\pm\! 0.0500$ & $2.4498\!\pm\! 0.0682$ & $2.3807\!\pm\! 0.0934$ & $3.5811\!\pm\! 0.1253$ & $3.8352\!\pm\! 0.0709$\\ \hline
\multirow{ 2}{*}{german} & AUC & $0.7938\!\pm\! 0.0246$ & $0.7943\!\pm\! 0.0255$ & $0.7879\!\pm\! 0.0326$ & $0.7932\!\pm\! 0.0313$ & $0.7890\!\pm\! 0.0278$ & $0.7933\!\pm\! 0.0262$\\    & Time & $0.0094\!\pm\! 0.0011$ & $0.0099\!\pm\! 0.0015$ & $0.0179\!\pm\! 0.0010$ & $0.0169\!\pm\! 0.0007$ & $0.0281\!\pm\! 0.0026$ & $0.0278\!\pm\! 0.0015$\\ \hline
\multirow{ 2}{*}{satimage} & AUC & $0.9772\!\pm\! 0.0029$ & $0.9769\!\pm\! 0.0040$ & $0.9765\!\pm\! 0.0028$ & $0.9300\!\pm\! 0.0066$ & $0.9760\!\pm\! 0.0029$ & $0.9770\!\pm\! 0.0041$\\    & Time & $0.0609\!\pm\! 0.0038$ & $0.0589\!\pm\! 0.0010$ & $0.1181\!\pm\! 0.0102$ & $0.1212\!\pm\! 0.0024$ & $0.1802\!\pm\! 0.0114$ & $0.1826\!\pm\! 0.0123$\\ \hline
\multirow{ 2}{*}{acoustic} & AUC & $0.8952\!\pm\! 0.0026$ & $0.8910\!\pm\! 0.0028$ & $0.8911\!\pm\! 0.0032$ & $0.8911\!\pm\! 0.0026$ & $0.8929\!\pm\! 0.0028$ & $0.8877\!\pm\! 0.0076$\\    & Time & $0.7281\!\pm\! 0.0216$ & $0.7304\!\pm\! 0.0278$ & $1.3972\!\pm\! 0.0168$ & $1.6672\!\pm\! 0.0213$ & $2.7367\!\pm\! 0.2275$ & $2.2063\!\pm\! 0.1009$\\ \hline
\multirow{ 2}{*}{covtype} & AUC & $0.8236\!\pm\! 0.0009$ & $0.8235\!\pm\! 0.0009$ & $0.8228\!\pm\! 0.0013$ & $0.8233\!\pm\! 0.0009$ & $0.8134\!\pm\! 0.0036$ & $0.8233\!\pm\! 0.0007$\\    & Time & $5.4320\!\pm\! 0.2447$ & $5.4988\!\pm\! 0.0629$ & $10.5169\!\pm\! 0.3023$ & $13.1290\!\pm\! 0.6334$ & $19.303\!\pm\! 3.6072$ & $16.064\!\pm\! 0.2350$\\ \hline
\multirow{ 2}{*}{a9a} & AUC & $0.9000\!\pm\! 0.0033$ & $0.9003\!\pm\! 0.0042$ & $0.9003\!\pm\! 0.0033$ & $0.8957\!\pm\! 0.0028$ & $0.8879\!\pm\! 0.0043$ & $0.9002\!\pm\! 0.0031$\\    & Time & $0.3123\!\pm\! 0.0018$ & $0.3120\!\pm\! 0.0023$ & $0.5862\!\pm\! 0.0035$ & $0.9417\!\pm\! 0.0610$ & $0.9686\!\pm\! 0.1143$ & $0.9273\!\pm\! 0.0146$\\ \hline
\multirow{ 2}{*}{connect} & AUC & $0.8786\!\pm\! 0.0031$ & $0.8783\!\pm\! 0.0023$ & $0.8783\!\pm\! 0.0032$ & $0.8716\!\pm\! 0.0027$ & $0.8583\!\pm\! 0.0035$ & $0.8776\!\pm\! 0.0036$\\    & Time & $0.6520\!\pm\! 0.0082$ & $0.6532\!\pm\! 0.0053$ & $1.2386\!\pm\! 0.0142$ & $1.9633\!\pm\! 0.0864$ & $2.0464\!\pm\! 0.2030$ & $2.0307\!\pm\! 0.0376$\\ \hline
\multirow{ 2}{*}{usps} & AUC & $0.9225\!\pm\! 0.0048$ & $0.9226\!\pm\! 0.0046$ & $0.9182\!\pm\! 0.0065$ & $0.8765\!\pm\! 0.0105$ & $0.9079\!\pm\! 0.0069$ & $0.9154\!\pm\! 0.0050$\\    & Time & $0.0947\!\pm\! 0.0044$ & $0.1026\!\pm\! 0.0022$ & $0.1821\!\pm\! 0.0093$ & $0.2851\!\pm\! 0.0282$ & $0.2638\!\pm\! 0.0195$ & $0.2949\!\pm\! 0.0045$\\ \hline
\multirow{ 2}{*}{w8a} & AUC & $0.9694\!\pm\! 0.0035$ & $0.9692\!\pm\! 0.0040$ & $0.9663\!\pm\! 0.0041$ & $0.9454\!\pm\! 0.0057$ & $0.9640\!\pm\! 0.0044$ & $0.9695\!\pm\! 0.0036$\\    & Time & $0.5414\!\pm\! 0.0069$ & $0.5401\!\pm\! 0.0262$ & $0.9725\!\pm\! 0.0119$ & $1.6080\!\pm\! 0.1558$ & $1.5541\!\pm\! 0.1302$ & $1.5782\!\pm\! 0.0228$\\ \hline
\multirow{ 2}{*}{mnist} & AUC & $0.9306\!\pm\! 0.0020$ & $0.9302\!\pm\! 0.0017$ & $0.9304\!\pm\! 0.0027$ & $0.8345\!\pm\! 0.0086$ & $0.8908\!\pm\! 0.0047$ & $0.9302\!\pm\! 0.0015$\\    & Time & $0.8272\!\pm\! 0.0241$ & $0.8409\!\pm\! 0.0168$ & $1.3983\!\pm\! 0.0592$ & $2.3366\!\pm\! 0.2359$ & $2.2333\!\pm\! 0.2533$ & $2.3376\!\pm\! 0.0418$\\ \hline
\multirow{ 2}{*}{gisette} & AUC & $0.9970\!\pm\! 0.0011$ & $0.9969\!\pm\! 0.0011$ & $0.9940\!\pm\! 0.0014$ & - & $0.9931\!\pm\! 0.0017$ & $0.9908\!\pm\! 0.0024$\\    & Time & $0.2899\!\pm\! 0.0224$ & $0.2846\!\pm\! 0.0208$ & $0.3291\!\pm\! 0.0253$ & - & $0.8719\!\pm\! 0.0807$ & $0.5778\!\pm\! 0.0423$\\ \hline
\multirow{ 2}{*}{real-sim} & AUC & $0.9955\!\pm\! 0.0004$ & $0.9959\!\pm\! 0.0002$ & $0.9936\!\pm\! 0.0005$ & - & $0.9842\!\pm\! 0.0021$ & $0.9934\!\pm\! 0.0006$\\    & Time & $8.6884\!\pm\! 0.2815$ & $9.5146\!\pm\! 0.3872$ & $8.9692\!\pm\! 0.3466$ & - & $25.505\!\pm\! 0.7707$ & $16.132\!\pm\! 0.4540$\\ \hline
\multirow{ 2}{*}{protein\_h} & AUC & $0.9858\!\pm\! 0.0029$ & $0.9806\!\pm\! 0.0030$ & $0.9807\!\pm\! 0.0054$ & $0.9825\!\pm\! 0.0040$ & $0.9895\!\pm\! 0.0017$ & $0.9793\!\pm\! 0.0036$\\    & Time & $1.1807\!\pm\! 0.0293$ & $1.1943\!\pm\! 0.0296$ & $2.2331\!\pm\! 0.0853$ & $4.4396\!\pm\! 0.4447$ & $3.5537\!\pm\! 0.7592$ & $3.5484\!\pm\! 0.1478$\\ \hline
\multirow{ 2}{*}{malware} & AUC & $0.9606\!\pm\! 0.0122$ & $0.9595\!\pm\! 0.0126$ & $0.9589\!\pm\! 0.0143$ & $0.9587\!\pm\! 0.0129$ & $0.9566\!\pm\! 0.0152$ & $0.9581\!\pm\! 0.0114$\\    & Time & $0.7291\!\pm\! 0.0212$ & $0.7296\!\pm\! 0.0183$ & $1.2822\!\pm\! 0.0473$ & $2.1483\!\pm\! 0.2648$ & $1.9903\!\pm\! 0.1021$ & $1.9604\!\pm\! 0.0637$\\ \hline
\multirow{ 2}{*}{webspam\_u} & AUC & $0.9673\!\pm\! 0.0008$ & $0.9664\!\pm\! 0.0007$ & $0.9668\!\pm\! 0.0005$ & $0.9659\!\pm\! 0.0006$ & $0.9670\!\pm\! 0.0012$ & $0.9671\!\pm\! 0.0006$\\    & Time & $3.7163\!\pm\! 0.2121$ & $3.3759\!\pm\! 0.1412$ & $6.0562\!\pm\! 0.3276$ & $12.3849\!\pm\! 0.9478$ & $9.9146\!\pm\! 0.1668$ & $9.7933\!\pm\! 0.5240$\\ \hline
\end{tabular}    		
\end{table}

\subsection{Experimental results\label{sec:result}}
In this section, we present the experimental results and discuss the comparisons of our algorithm against other ones. In Figure \ref{fig:auc-all},
we plot the AUC values of the constructed models on the test data versus execution time in seconds for SPAUC (without regularization), SPAM, SOLAM, OPAUC, OAM\_gra and FSAUC. It is observed that
SPAUC attains a faster training speed than all baseline methods.

In particular, the curve of SOLAM fluctuates rapidly, especially in the early stage of the optimization, which is perhaps due to the requirement
of updating both primal and dual variables.
OAM\_gra behaves more robustly, which, however, requires a high computation burden due to the requirement in updating a buffer and comparing the
current example and examples in the buffer per iteration.
As one can see from Figure \ref{fig:auc-all}, SPAUC converges faster than FSAUC on most of the datasets. The underlying reason could be two-fold. Firstly, FSAUC requires a projection onto the intersection of an $\ell_1$-ball and $\ell_2$-ball which requires an alternating projection step.
Secondly, FSAUC requires to update both primal and dual variables, which further increases the computational cost per iteration.
OPAUC has a low training speed due to the requirement in handling a covariance matrix,
which is especially unfavorable for high-dimensional datasets. For example, OPAUC has the slowest training speed on USPS for which the dimensionality is $256$.
We do not run OPAUC on datasets with dimensionality larger than $1000$ due to the heavy dependency of its time complexity on the dimensionality.
The implementation of SPAM requires an accurate information of $p,\ebb[x|y=1]$ and $\ebb[x|y=-1]$,
which we approximate with
\begin{equation}\label{approximation}
\hat{p}=\frac{\sum_{i=1}^{n}\ibb_{[y_i=1]}}{n},\quad
\hat{u}=\frac{\sum_{i=1}^{n}x_i\ibb_{[y_i=1]}}{\sum_{i=1}^{n}\ibb_{[y_i=1]}},\quad
\hat{v}=\frac{\sum_{i=1}^{n}x_i\ibb_{[y_i=-1]}}{\sum_{i=1}^{n}\ibb_{[y_i=-1]}}.
\end{equation}
It is observed that the AUC curve for SPAM attains a sharp increase at the beginning of the curve and then moderately increases.
The underlying reason is that we include the computational cost of calculating $\hat{p},\hat{u}$ and $\hat{v}$ in the curve, which requires to go through the whole training set.

In Table \ref{tab:auc-all}, we also report detailed AUCs as well as the execution time per pass, both in the form of mean $\pm$ standard deviation.  We can see from Table \ref{tab:auc-all} that SPAUC achieves accuracies comparable to the state-of-the-art methods over all datasets. SPAUC (without regularization) and SPAM require comparable running time per iteration since both algorithms require
no projections and no updates on the dual variables.
An advantage of SPAUC with no regularization over SPAM is that SPAUC can deal with streaming data in a truly online fashion,  while SPAM needs to know the conditional expectations in \eqref{approximation} and hence is not an online learning algorithm.
Furthermore, the fast convergence of SPAM requires the objective function to be strongly convex~\citep{natole2018stochastic}, which introduces an additional regularization parameter to tune.
Other baseline methods require longer per-pass running time due to the same reasons we mentioned above for explaining the AUC curve in Figure \ref{fig:auc-all}.
It can be seen that OAM\_gra requires longer per-pass running time than OPAUC if the dimensionality is relatively small, while the reverse is the case for datasets with a relatively large dimensionality.
This is consistent with the dependency of the time complexity on the dimensionality for these two methods, i.e., linear versus quadratic.

\begin{figure}[htb]
  \centering
  \hspace*{-1.2cm}
  \subfigure[diabetes]{\includegraphics[width=0.306\textwidth, trim=0 2 0 2, clip]{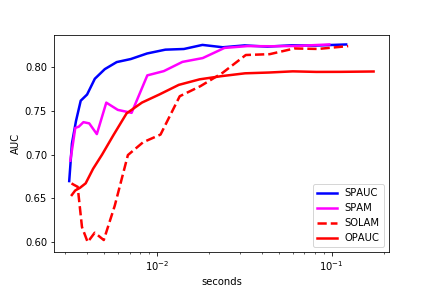}}\hspace*{-0.342cm}
  \subfigure[ijcnn1]{\includegraphics[width=0.306\textwidth, trim=0 2 0 2, clip]{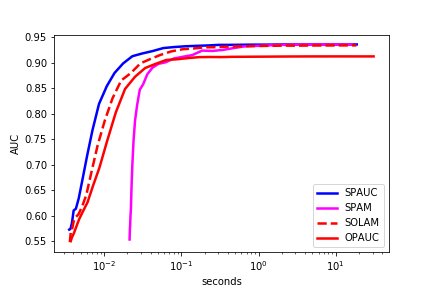}}\hspace*{-0.342cm}
  \subfigure[german]{\includegraphics[width=0.306\textwidth, trim=0 2 0 2, clip]{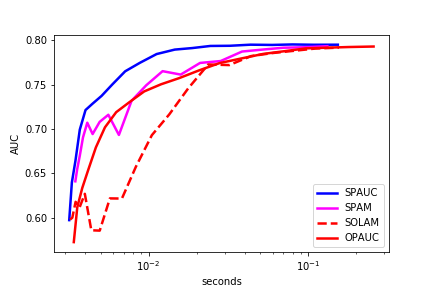}}\hspace*{-0.342cm}
  \subfigure[satimage]{\includegraphics[width=0.306\textwidth, trim=0 2 0 2, clip]{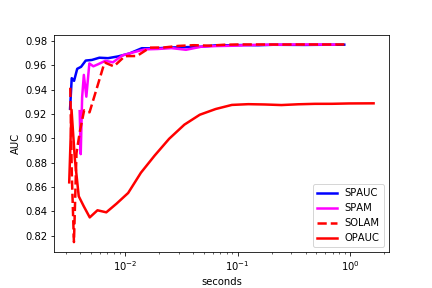}}\hspace*{-0.342cm}\\

  \hspace*{-1.2cm}
  \subfigure[acoustic]{\includegraphics[width=0.306\textwidth, trim=0 2 0 2, clip]{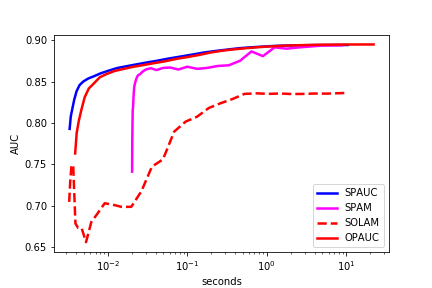}}\hspace*{-0.342cm}
  \subfigure[covtype]{\includegraphics[width=0.306\textwidth, trim=0 2 0 2, clip]{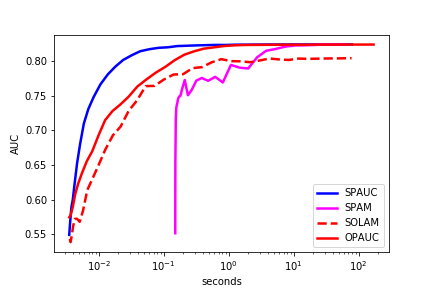}}\hspace*{-0.342cm}
  \subfigure[a9a]{\includegraphics[width=0.306\textwidth, trim=0 2 0 2, clip]{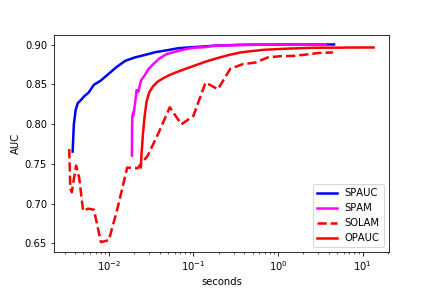}}\hspace*{-0.342cm}
  \subfigure[connect]{\includegraphics[width=0.306\textwidth, trim=0 2 0 2, clip]{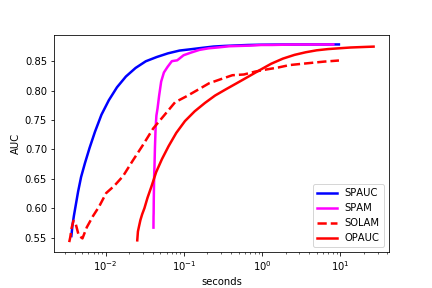}}\hspace*{-0.342cm}\\

  \hspace*{-1.2cm}
  \subfigure[usps]{\includegraphics[width=0.306\textwidth, trim=0 2 0 2, clip]{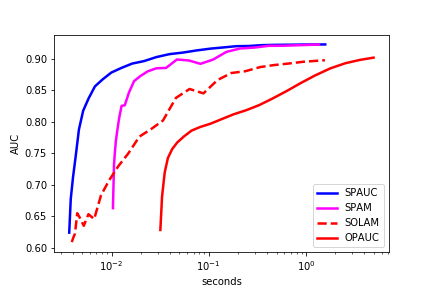}}\hspace*{-0.342cm}
  \subfigure[w8a]{\includegraphics[width=0.306\textwidth, trim=0 2 0 2, clip]{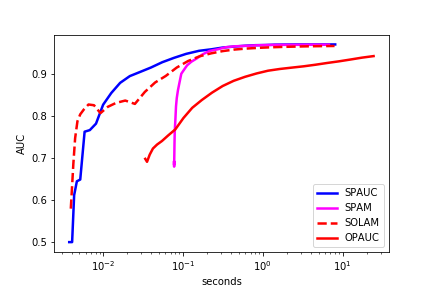}}\hspace*{-0.342cm}
  \subfigure[mnist]{\includegraphics[width=0.306\textwidth, trim=0 2 0 2, clip]{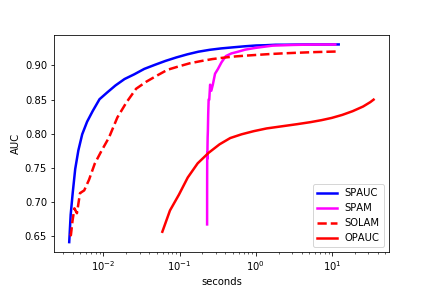}}\hspace*{-0.342cm}
  \subfigure[gisette]{\includegraphics[width=0.306\textwidth, trim=0 2 0 2, clip]{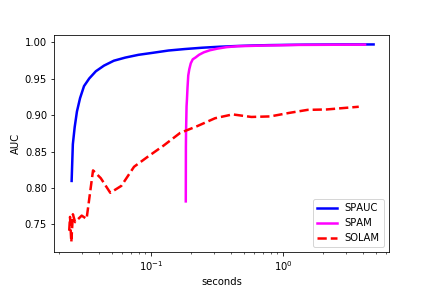}}\hspace*{-0.342cm}\\

  \hspace*{-1.2cm}
  \subfigure[real-sim]{\includegraphics[width=0.306\textwidth, trim=0 2 0 2, clip]{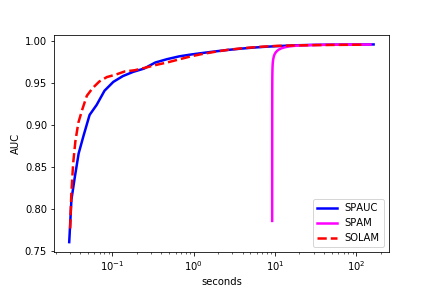}}\hspace*{-0.342cm}
  \subfigure[protein\_h]{\includegraphics[width=0.306\textwidth, trim=0 2 0 2, clip]{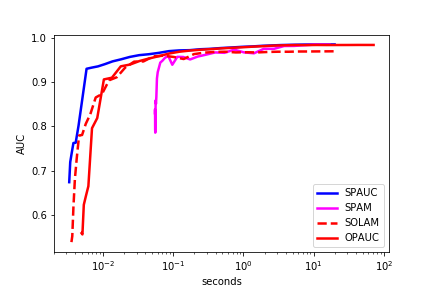}}\hspace*{-0.342cm}
  \subfigure[malware]{\includegraphics[width=0.306\textwidth, trim=0 2 0 2, clip]{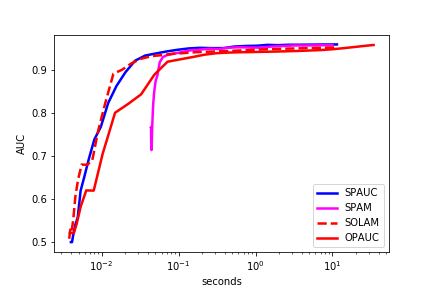}}\hspace*{-0.342cm}
  \subfigure[webspam\_u]{\includegraphics[width=0.306\textwidth, trim=0 2 0 2, clip]{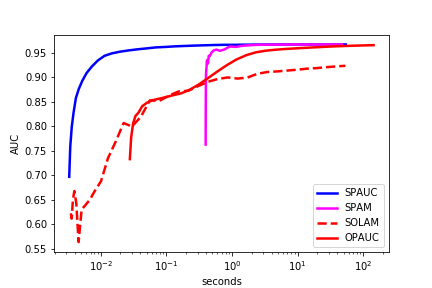}}\hspace*{-0.342cm}
  \vspace*{-0.6\baselineskip}
  \caption{AUC versus time curves (in seconds) for SPAUC, SPAM, SOLAM and OPAUC for objective functions with regularization parameter $\lambda=10^{-6}$. \label{fig:auc-l2-6}}
\end{figure}

To show that SPAUC also works well with regularization, we consider \eqref{SAUC} with $\Omega(\bw)=\lambda\|\bw\|_2^2$ in our experiments.
We compare SPAUC with this $\ell_2$-regularizer to several baseline methods including SPAM, SOLAM and OPAUC, where we modify the original SOLAM in \citet{YWL} by replacing the $\ell_2$-constraint with an $\ell_2$-regularizer. Therefore, these four methods all optimize the same objective function with an $\ell_2$-regularizer.
We fix the regularization parameter and tune the step-size parameter $\mu$ by 5-fold cross validation.
In Figure \ref{fig:auc-l2-6}, we plot the AUC values as a function of execution time (in seconds) for SPAUC (with $\ell_2$-regularizer), SPAM, SOLAM and OPAUC with $\lambda=10^{-6}$.
It can be seen that SPAUC with $\ell_2$-regularizer attains a fast convergence speed as compared to the baseline methods. The same phenomenon also occurs for other choice of regularization parameters, e.g., $\lambda=10^{-2}$ and $\lambda=10^{-4}$. We omit these results to save space.

\section{Proofs\label{sec:proof}}


In this section, we present proofs for theoretical properties of SPAUC. In subsection \ref{sec:elem}, we present several useful properties on the objective function which will be useful in our convergence rate analysis.
Then we move on to the one-step progress inequality of SPAUC together with some useful corollaries. Subsection \ref{sec:approx} presents high-probability bounds on approximating
$\widetilde{F}'(\bw;z_t)$ by $\hat{F}_t'(\bw;z_t)$, based on which we establish an almost boundedness of  iterates in subsection \ref{sec:boundedness}. In subsection \ref{sec:proof-rate} and  subsection \ref{sec:proof-rate-fast},
we use these preliminary results to prove convergence rates for SPAUC applied to general convex AUC objectives and AUC objectives with a quadratic functional growth, respectively.

\subsection{Properties of Objective Functions\label{sec:elem}}

The following lemma shows that an approximation of $p,\ebb[x|y=1]$ and $\ebb[x|y=-1]$ by \eqref{approx} still preserves the convexity. It also establishes the self-bounding property of $\hat{F}_t(\bw;z)$.
\begin{lemma}\label{lem:self-bounding-F}
For any $\bw$ and $z$, we have
\begin{equation}\label{self-bounding-F}
  \|\hat{F}_t'(\bw;z)\|_2^2\leq 16\kappa^2\hat{F}_t(\bw;z)\quad\text{and}\quad \hat{F}_t(\bw;z)\geq0.
\end{equation}
Furthermore, for any $z$ the function $\hat{F}_t(\bw;z)$ is a convex function of $\bw$.
\end{lemma}
\begin{proof}
The inequality $\hat{F}_t(\bw;z)\geq0$ follows directly from the Schwartz's inequality:
$$
\hat{F}_t(\bw;z)\geq 2p_t(1-p_t)\bw^\top\big(v_t-u_t\big)+p_t(1-p_t)\big(\bw^\top\big(v_t-u_t\big)\big)^2+p_t(1-p_t)\geq 0.
$$
For any $\bw$ and $\tilde{\bw}$, we have
\begin{multline*}
\big\|\hat{F}'_t(\bw;z)-\hat{F}'_t(\tilde{\bw};z)\big\|_2
\leq2(1-p_t)\big\|(x-u_t)(x-u_t)^\top(\bw-\tilde{\bw})\big\|_2\ibb_{[y=1]}\\+
2p_t\big\|(x-v_t)(x-v_t)^\top(\bw-\tilde{\bw})\big\|_2\ibb_{[y=-1]}+\\
2p_t(1-p_t)\big\|\big(v_t-u_t\big)\big(v_t-u_t\big)^\top(\bw-\tilde{\bw})\big\|_2
\leq 8\kappa^2\|\bw-\tilde{\bw}\|_2,
\end{multline*}
where in the last inequality we have used the definition of $\kappa$.

Therefore, it follows from the self-bounding property of non-negative smooth functions (Lemma \ref{lem:self-bounding}) that
$
\|\hat{F}'_t(\bw;z)\|_2^2\leq 16\kappa^2\hat{F}_t(\bw;z).
$
This establishes \eqref{self-bounding-F}.

It is clear that the Hessian matrix of $\hat{F}_t(\bw;z)$ is
$$2(1-p_t)\big(x-u_t\big)\big(x-u_t\big)^\top\ibb_{[y=1]}+2p_t\big(x-v_t\big)\big(x-v_t\big)^\top\ibb_{[y=-1]}+2p_t(1-p_t)\big(v_t-u_t\big)\big(v_t-u_t\big)^\top,$$
which is a semi-positive definite matrix. Therefore, $\hat{F}_t(\cdot;z)$ is a convex function for any $z$.
The proof is complete.
\end{proof}

\subsection{One-step Progress Inequality with Useful Corollaries\label{sec:osp}}
Our theoretical analysis roots its foundation on the following one-step progress inequality measuring how the iterate would change after a single iteration of \eqref{SAUC}. 
\begin{lemma}\label{lem:osp}
  Let $\{\bw_t\}_t$ be produced by \eqref{SAUC}. If Assumption \ref{ass:self-bounding} holds, then for any $\bw\in\rbb^d$ we have
  \begin{multline}\label{osp}
    \|\bw_{t+1}-\bw\|_2^2 - \|\bw-\bw_t\|_2^2 \leq 2\eta_t\big\langle\bw-\bw_t,\hat{F}'_t(\bw_t;z_t)\rangle+2\eta_t\big(\Omega(\bw)-\Omega(\bw_{t})\big)\\
  -\eta_t\sigma_\Omega\|\bw-\bw_{t+1}\|_2^2+2\eta_t^2\big(C_1\hat{F}_t(\bw_t;z_t)+C_1\Omega(\bw_t)+A_2\big).
  \end{multline}
\end{lemma}
\begin{proof}
According to the first-order optimality condition in \eqref{SAUC}, we get
\begin{equation}\label{first-order}
  \eta_t\hat{F}'_t(\bw_t;z_t)+\eta_t\Omega'(\bw_{t+1})+\bw_{t+1}-\bw_t=0,
\end{equation}
from which we derive
\begin{align}
  \|\bw_{t+1}-\bw\|_2^2 &= \langle\bw_{t+1}-\bw,\bw_{t+1}-\bw_t+\bw_t-\bw\rangle\notag\\
  & = -\eta_t\big\langle\bw_{t+1}-\bw,\hat{F}'_t(\bw_t;z_t)\rangle+\eta_t\langle\bw-\bw_{t+1},\Omega'(\bw_{t+1})\rangle+\langle\bw_{t+1}-\bw,\bw_t-\bw\rangle.\label{osp-1}
\end{align}
It follows from the definition of $\sigma_\Omega$ that
\begin{align}\label{osp-2}
  & \langle\bw-\bw_{t+1},\Omega'(\bw_{t+1})\rangle \leq \Omega(\bw)-\Omega(\bw_{t+1})-2^{-1}\sigma_\Omega\|\bw-\bw_{t+1}\|_2^2\notag\\
  & = \Omega(\bw)-\Omega(\bw_{t})+\Omega(\bw_{t})-\Omega(\bw_{t+1})-2^{-1}\sigma_\Omega\|\bw-\bw_{t+1}\|_2^2\notag\\
  & \leq \Omega(\bw)-\Omega(\bw_{t})+\langle\bw_t-\bw_{t+1},\Omega'(\bw_t)\rangle-2^{-1}\sigma_\Omega\big(\|\bw-\bw_{t+1}\|_2^2+\|\bw_t-\bw_{t+1}\|_2^2\big).
\end{align}
It can be directly checked that
\[
  \langle\bw_{t+1}-\bw,\bw_t-\bw\rangle=\frac{1}{2}\Big(\|\bw-\bw_t\|_2^2+\|\bw-\bw_{t+1}\|_2^2-\|\bw_t-\bw_{t+1}\|^2\Big).
\]
Plugging the above identity and \eqref{osp-2} back into \eqref{osp-1}, we derive
\begin{multline}\label{osp-3}
  \|\bw_{t+1}-\bw\|_2^2  \leq  \eta_t\big\langle\bw-\bw_t+\bw_t-\bw_{t+1},\hat{F}'_t(\bw_t;z_t)\rangle+\eta_t\Omega(\bw)-\eta_t\Omega(\bw_{t})+\eta_t\langle\bw_t-\bw_{t+1},\Omega'(\bw_t)\rangle\\
  -2^{-1}\eta_t\sigma_\Omega\|\bw-\bw_{t+1}\|_2^2+\frac{1}{2}\Big(\|\bw-\bw_t\|_2^2+\|\bw-\bw_{t+1}\|_2^2-\|\bw_t-\bw_{t+1}\|^2\Big).
\end{multline}
According to the Schwartz's inequality, we know
\[
\eta_t\big\langle\bw_t-\bw_{t+1},\hat{F}'_t(\bw_t;z_t)\rangle+\eta_t\langle\bw_t-\bw_{t+1},\Omega'(\bw_t)\rangle\leq \frac{1}{2}\|\bw_t-\bw_{t+1}\|_2^2+\eta_t^2\|\hat{F}'_t(\bw_t;z_t)\|_2^2+\eta_t^2\|\Omega'(\bw_t)\|_2^2.
\]
Plugging the above inequality back into \eqref{osp-3} gives
\begin{multline}\label{osp-4}
  \|\bw_{t+1}-\bw\|_2^2 - \|\bw-\bw_t\|_2^2 \leq 2\eta_t\big\langle\bw-\bw_t,\hat{F}'_t(\bw_t;z_t)\rangle+2\eta_t\big(\Omega(\bw)-\Omega(\bw_{t})\big)\\
  -\eta_t\sigma_\Omega\|\bw-\bw_{t+1}\|_2^2+2\eta_t^2\|\hat{F}'_t(\bw_t;z_t)\|_2^2+2\eta_t^2\|\Omega'(\bw_t)\|_2^2.
\end{multline}
The stated bound then follows from Lemma \ref{lem:self-bounding-F}, Assumption \ref{ass:self-bounding} and the definition of $C_1$.
The proof is complete.
\end{proof}

Based on Lemma \ref{lem:osp}, we can derive several useful inequalities collected in the following corollary.
Eq. \eqref{boundness-a} provides a general bound on the norm of iterates in terms of step sizes.
Eqs. \eqref{boundness-b} and \eqref{boundness-c} show how the accumulation of function values can be controlled by step sizes,
which, according to Lemma \ref{lem:self-bounding-F} and Assumption \ref{ass:self-bounding}, in turn give useful estimates on $\sum_{k=1}^{t}\eta_k^2\big(\|\hat{F}_k'(\bw_k,z_k)\|_2^2+\|\Omega'(\bw_k)\|_2^2\big)$
and $\sum_{k=1}^{t}\big(\|\hat{F}_k'(\bw_k,z_k)\|_2^2+\|\Omega'(\bw_k)\|_2^2\big)$ required to handle in convergence analysis.
\begin{corollary}\label{lem:boundness}
  Let $\{\bw_t\}_t$ be produced by \eqref{SAUC}. Suppose $\eta_t\leq (2C_1)^{-1}$ and Assumption \ref{ass:self-bounding} holds. Let
  $C_4:=C_1^{-1}A_2+2^{-1}$.
  Then
  \begin{equation}\label{boundness-a}
    \|\bw_{t+1}\|_2^2\leq C_4\sum_{k=1}^{t}\eta_k.
  \end{equation}
  Furthermore, if $\eta_{t+1}\leq \eta_t$, then
  \begin{equation}\label{boundness-b}
    \sum_{k=1}^{t}\eta_k^2\big(\hat{F}_k(\bw_k;z_k)+\Omega(\bw_k)\big)\leq C_4\sum_{k=1}^{t}\eta_k^2
  \end{equation}
  and
  \begin{equation}\label{boundness-c}
      \sum_{k=1}^t\big(\hat{F}_k(\bw_k;z_k)+\Omega(\bw_k)\big)\leq C_4t+C_4\eta_t^{-1}\sum_{k=1}^t\eta_k.
  \end{equation}
\end{corollary}
\begin{proof}
Eq. \eqref{osp} together with the convexity of $\hat{F}_t$ established in Lemma \ref{lem:self-bounding-F} implies
\begin{multline}\label{boundedness-0}
  \|\bw_{t+1}-\bw\|_2^2 - \|\bw_t-\bw\|_2^2 \leq 2\eta_t\big(\hat{F}_t(\bw;z_t)-\hat{F}_t(\bw_t;z_t)\big)+2\eta_t\big(\Omega(\bw)-\Omega(\bw_{t})\big)\\
  -\eta_t\sigma_\Omega\|\bw-\bw_{t+1}\|_2^2+2\eta_t^2\big(C_1\hat{F}_t(\bw_t;z_t)+C_1\Omega(\bw_t)+A_2\big).
\end{multline}
Taking $\bw=0$ in \eqref{boundedness-0} and using $\hat{F}_t(0;z_t)=p_t(1-p_t), \Omega(0)=0$, we get
\begin{align}
  & \|\bw_{t+1}\|_2^2 - \|\bw_t\|_2^2\notag\\
  &\leq 2\eta_t\big(\hat{F}_t(0;z_t)-\hat{F}_t(\bw_t;z_t)\big)+2\eta_t\big(\Omega(0)-\Omega(\bw_{t})\big)
  +2\eta_t^2\big(C_1\hat{F}_t(\bw_t;z_t)+C_1\Omega(\bw_t)+A_2\big)\notag\\
  & \leq 2\eta_t(C_1\eta_t-1)\big(\hat{F}_t(\bw_t;z_t)+\Omega(\bw_t)\big)+\eta_t/2+2\eta_t^2A_2\label{boundedness-1}\\
  & \leq  \eta_t/2+C_1^{-1}A_2\eta_t,\notag
\end{align}
where the last inequality follows from $\hat{F}_t(\bw_t;z_t)+\Omega(\bw_t)\geq0$ due to Lemma \ref{lem:self-bounding-F} and the assumption $0\leq\eta_t\leq (2C_1)^{-1}$.
Taking a summation of the above inequality then shows
\[
  \|\bw_{t+1}\|_2^2\leq \big(C_1^{-1}A_2+2^{-1}\big)\sum_{k=1}^{t}\eta_k.
\]
This establishes \eqref{boundness-a}.
Plugging the assumption $\eta_t\leq(2C_1)^{-1}$ into \eqref{boundedness-1} gives
$$
\eta_t\big(\hat{F}_t(\bw_t;z_t)+\Omega(\bw_t)\big)\leq \|\bw_t\|_2^2-\|\bw_{t+1}\|_2^2+\eta_t/2+C_1^{-1}A_2\eta_t.
$$
Multiplying both sides by $\eta_t$, we derive
\begin{align*}
\eta_t^2\big(\hat{F}_t(\bw_t;z_t)+\Omega(\bw_t)\big) &\leq \eta_t\|\bw_t\|_2^2-\eta_t\|\bw_{t+1}\|_2^2+\eta_t^2/2+\eta_t^2C_1^{-1}A_2\\
& \leq \eta_t\|\bw_t\|_2^2-\eta_{t+1}\|\bw_{t+1}\|_2^2+\eta_t^2/2+\eta_t^2C_1^{-1}A_2,
\end{align*}
where we have used the assumption $\eta_{t+1}\leq \eta_t$.
Taking a summation of the above inequality further yields
\[
  \sum_{k=1}^{t}\eta_k^2\big(\hat{F}_k(\bw_k;z_k)+\Omega(\bw_k)\big) \leq \big(C_1^{-1}A_2+2^{-1}\big)\sum_{k=1}^{t}\eta_k^2
\]
We now turn to \eqref{boundness-c}.
Plugging the assumption $\eta_t\leq(2C_1)^{-1}$ into \eqref{boundedness-1} and multiplying both sides by $\eta_t^{-1}$, we derive
\[
  \hat{F}'(\bw_t;z_t)+\Omega(\bw_t)\leq \eta_t^{-1}\big(\|\bw_t\|_2^2-\|\bw_{t+1}\|_2^2\big)+2^{-1}+C_1^{-1}A_2.
\]
Taking a summation of the above inequality implies
\begin{align*}
  \sum_{k=1}^t\big(\hat{F}_k(\bw_k;z_k)+\Omega(\bw_k)\big)&\leq tC_4+\sum_{k=1}^t\eta_k^{-1}\big(\|\bw_k\|_2^2-\|\bw_{k+1}\|_2^2\big)\\
  & \leq tC_4+\sum_{k=2}^t\|\bw_k\|_2^2(\eta_k^{-1}-\eta_{k-1}^{-1})+\eta_1^{-1}\|\bw_1\|_2^2\\
  &\leq tC_4+\max_{1\leq\tilde{k}\leq t}\|\bw_{\tilde{k}}\|_2^2\sum_{k=2}^t(\eta_k^{-1}-\eta_{k-1}^{-1})\\
  & \leq tC_4+C_4\eta_t^{-1}\sum_{k=1}^t\eta_k,
\end{align*}
where the last inequality is due to \eqref{boundness-a}.
The proof is complete.
\end{proof}

\subsection{Approximation of Stochastic Gradients\label{sec:approx}}
The implementation of SPAUC requires to approximate the unbiased stochastic gradient $\widetilde{F}'(\bw_t;z_t)$ by replacing the involved
$p,\ebb[x|y=1],\ebb[x|y=-1]$ with their empirical counterparts.
The following lemma gives a quantitative measure on the accuracy of this approximation.
\begin{lemma}\label{lem:approx}
  Let $\delta\in(0,1)$. For any $t\in\nbb$, the following inequality holds with probability at least $1-\delta$
  $$
     \big\| \widetilde{F}'(\bw_t;z_t)-\hat{F}_t'(\bw_t;z_t)\big\|_2\leq \frac{2\kappa^2\big(2+\sqrt{2\log(3/\delta)}\big)}{\sqrt{t}}\Big(C_p\|\bw_t\|_2+3\Big),
  $$
  where $C_p=8\max\{p^{-1},(1-p)^{-1}\}+24$.
\end{lemma}
Before proving Lemma \ref{lem:approx}, we need to introduce the following preliminary lemma. For a matrix $A$, we denote by $\|A\|_{\mathrm{op}}$ the operator norm of $A$, i.e., $\|A\|_{\mathrm{op}}=\sup_{\|\bw\|_2=1}\|A\bw\|_2$. For any $u,v\in\rbb^d$, there holds
\begin{equation}\label{op-nm}
  \|uv^\top\|_{\mathrm{op}}\leq \|u\|_2\|v\|_2.
\end{equation}
\begin{lemma}\label{lem:concentration-auc}
  Let $\delta\in(0,1)$. For any $t\in\nbb$, with probability at least $1-\delta$ the following inequalities hold simultaneously for all $\bw\in\rbb^d$
  \begin{gather}
    |p-p_t|\leq \big(2+\sqrt{2\log(3/\delta)}\big)/\sqrt{t},\label{concentration-auc-a}\\
    \|\ebb[x|y=1]-u_t\|_2\leq \frac{2\kappa(2+\sqrt{2\log(3/\delta)})}{p\sqrt{t}},\label{concentration-auc-b}\\
    \|\ebb[x|y=-1]-v_t\|_2\leq \frac{2\kappa(2+\sqrt{2\log(3/\delta)})}{(1-p)\sqrt{t}},\label{concentration-auc-c}
  \end{gather}
  \begin{multline}
    \Big\|(1-p_t)(x-u_t)(x-u_t)^\top-(1-p)\big(x-\ebb[x|y=1]\big)\big(x-\ebb[x|y=1]\big)^\top\Big\|_{\mathrm{op}} \\
    \leq \frac{8\kappa^2\big(2+\sqrt{2\log(3/\delta)}\big)}{p\sqrt{t}},\label{concentration-auc-e}
  \end{multline}
  \begin{multline}
    \Big\|p_t(x-v_t)(x-v_t)^\top-p\big(x-\ebb[x|y=-1]\big)\big(x-\ebb[x|y=-1]\big)^\top\Big\|_{\mathrm{op}} \\
    \leq \frac{8\kappa^2\big(2+\sqrt{2\log(3/\delta)}\big)}{(1-p)\sqrt{t}},\label{concentration-auc-f}
  \end{multline}
  \begin{gather}
    \big\|p(1-p)\big(\ebb[x'|y'=-1]-\ebb[x|y=1]\big)-p_t(1-p_t)(v_t-u_t)\big\|_2\leq 3\kappa\big(2+\sqrt{2\log(3/\delta)}\big)/\sqrt{t},\label{concentration-auc-d}
  \end{gather}
  \begin{multline}
    \Big\|p(1-p)\big(\ebb[x'|y'=-1]-\ebb[x|y=1]\big)\big(\ebb[x'|y'=-1]-\ebb[x|y=1]\big)^\top-\\
    p_t(1-p_t)(v_t-u_t)(v_t-u_t)^\top\Big\|_{\mathrm{op}}
    \leq\frac{24\kappa^2\big(2+\sqrt{2\log(3/\delta)}\big)}{\sqrt{t}}.
    \label{concentration-auc-g}
  \end{multline}
\end{lemma}
\begin{proof}
  According to Lemma \ref{lem:hoeffding}, with probability at least $1-\delta$ the following three inequalities hold simultaneously
  \begin{gather}
    |p-p_t|\leq \frac{2+\sqrt{2\log(3/\delta)}}{\sqrt{t}},\notag\\
    \big\|\ebb[x\ibb_{[y=1]}]-\frac{1}{t}\sum_{i=0}^{t-1}x_i\ibb_{[y_i=1]}\big\|_2\leq \frac{(2+\sqrt{2\log(3/\delta)})\kappa}{\sqrt{t}},\label{concentration-auc-01}\\
    \big\|\ebb[x\ibb_{[y=-1]}]-\frac{1}{t}\sum_{i=0}^{t-1}x_i\ibb_{[y_i=-1]}\big\|_2\leq \frac{(2+\sqrt{2\log(3/\delta)})\kappa}{\sqrt{t}}.\label{concentration-auc-02}
  \end{gather}

  \medskip
  We now prove \eqref{concentration-auc-b}. According to \eqref{approx}, we know
  \begin{align*}
    \|\ebb[x|y=1]-u_t\|_2&=\frac{1}{p}\Big\|p\ebb[x|y=1]-p_tu_t+p_tu_t-pu_t\Big\|_2 \\
    &\leq\frac{1}{p}\Big\|\ebb[x\ibb_{[y=1]}]-\frac{1}{t}\sum_{i=0}^{t-1}x_i\ibb_{[y_i=1]}\Big\|_2+\frac{\|u_t\|_2}{p}|p_t-p|.
  \end{align*}
  Then we can apply \eqref{concentration-auc-a} and \eqref{concentration-auc-01} to derive \eqref{concentration-auc-b} with probability at least $1-\delta$.

  \medskip

  Eq. \eqref{concentration-auc-c} can be proved in a similar manner and we omit the proof for brevity.

  \medskip

  We now show \eqref{concentration-auc-e}. It is clear that
  \begin{multline*}
    (1-p_t)(x-u_t)(x-u_t)^\top-(1-p)\big(x-\ebb[x|y=1]\big)\big(x-\ebb[x|y=1]\big)^\top=\big((1-p_t)-(1-p)\big)(x-u_t)(x-u_t)^\top\\+
    (1-p)(x-u_t)(x-u_t)^\top-(1-p)(x-u_t)\big(x-\ebb[x|y=1]\big)^\top\\
    +(1-p)(x-u_t)\big(x-\ebb[x|y=1]\big)^\top-(1-p)\big(x-\ebb[x|y=1]\big)\big(x-\ebb[x|y=1]\big)^\top,
  \end{multline*}
  from which and \eqref{op-nm} we derive
  \begin{multline*}
    \Big\|(1-p_t)(x-u_t)(x-u_t)^\top-(1-p)\big(x-\ebb[x|y=1]\big)\big(x-\ebb[x|y=1]\big)^\top\Big\|_{\mathrm{op}} \leq |p-p_t|\Big\|\big(x-u_t\big)\big(x-u_t\big)^\top\Big\|_{\mathrm{op}}\\
    +(1-p)\Big\|(x-u_t)\big(\ebb[x|y=1]-u_t\big)^\top\Big\|_{\mathrm{op}}+(1-p)\Big\|\big(\ebb[x|y=1]-u_t\big)\big(x-\ebb[x|y=1]\big)^\top\Big\|_{\mathrm{op}}\\
    \leq 4\kappa^2|p-p_t|+4\kappa(1-p)\|\ebb[x|y=1]-u_t\|_2.
  \end{multline*}
  This together with \eqref{concentration-auc-a} and \eqref{concentration-auc-b} shows \eqref{concentration-auc-e} with probability at least $1-\delta$.


  \medskip

  Eq. \eqref{concentration-auc-f} can be proved in a similar manner and we omit the proof for brevity.

  \medskip

  We now prove \eqref{concentration-auc-d}.
  It is clear
  \begin{multline*}
  p(1-p)\big(\ebb[x'|y'=-1]-\ebb[x|y=1]\big)-p_t(1-p_t)(v_t-u_t)=p\ebb[x'\ibb_{[y'=-1]}]
  -(1-p)\ebb[x\ibb_{[y=1]}]\\-\frac{p_t}{t}\Big(\sum_{i=0}^{t-1}x_i\ibb_{[y_i=-1]}\Big)+
  \frac{1-p_t}{t}\Big(\sum_{i=0}^{t-1}x_i\ibb_{[y_i=1]}\Big),
  \end{multline*}
  from which we derive
  \begin{align*}
    & \big\|p(1-p)\big(\ebb[x'|y'=-1]-\ebb[x|y=1]\big)-p_t(1-p_t)(v_t-u_t)\big\|_2  \\
    & \leq \Big\|p\ebb[x'\ibb_{[y'=-1]}]-\frac{p_t}{t}\Big(\sum_{i=0}^{t-1}x_i\ibb_{[y_i=-1]}\Big)\Big\|_2+\Big\|(1-p)\ebb[x\ibb_{[y=1]}]-
    \frac{(1-p_t)}{t}\Big(\sum_{i=0}^{t-1}x_i\ibb_{[y_i=1]}\Big)\Big\|_2 \\
    & \leq 2\kappa|p\!-\!p_t|+p_t\Big\|\ebb[x'\ibb_{[y'\!=\!-1]}]-\frac{1}{t}\Big(\sum_{i=0}^{t-1}x_i\ibb_{[y_i\!=\!-1]}\Big)\Big\|_2+(1\!-\!p_t)\Big\|\ebb[x\ibb_{[y=1]}]-\frac{1}{t}\Big(\sum_{i=0}^{t-1}x_i\ibb_{[y_i=1]}\Big)\Big\|_2,
  \end{align*}
  where we have used $p\ebb[x'\ibb_{[y'\!=\!-1]}]=(p-p_t)\ebb[x'\ibb_{[y'\!=\!-1]}]+p_t\ebb[x'\ibb_{[y'\!=\!-1]}]$.
  We can then apply \eqref{concentration-auc-a}, \eqref{concentration-auc-01} and \eqref{concentration-auc-02} to derive the bound \eqref{concentration-auc-d} with probability $1-\delta$.

  \medskip

  We now prove \eqref{concentration-auc-g}. It is clear
  \begin{align*}
    p(1-p)&\big(\ebb[x'|y'=-1]-\ebb[x|y=1]\big)\big(\ebb[x'|y'=-1]-\ebb[x|y=1]\big)^\top-p_t(1-p_t)(v_t-u_t)(v_t-u_t)^\top \\
    &=p(1-p)\big(\ebb[x'|y'=-1]-\ebb[x|y=1]\big)\Big(\big(\ebb[x'|y'=-1]-\ebb[x|y=1]\big)^\top-\big(v_t-u_t\big)^\top\Big)\\
    &+p(1-p)\Big(\big(\ebb[x'|y'=-1]-\ebb[x|y=1]\big)-\big(v_t-u_t\big)\Big)\big(v_t-u_t\big)^\top\\
    &+\big(p(1-p)-p_t(1-p_t)\big)\big(v_t-u_t\big)\big(v_t-u_t\big)^\top,
  \end{align*}
  from which and  \eqref{op-nm} it follows that
  \begin{multline*}
    \Big\|p(1-p)\big(\ebb[x'|y'=-1]-\ebb[x|y=1]\big)\big(\ebb[x'|y'=-1]-\ebb[x|y=1]\big)^\top-p_t(1-p_t)(v_t-u_t)(v_t-u_t)^\top\Big\|_{\mathrm{op}} \\
    \leq 4p(1-p)\kappa\Big\|\big(\ebb[x'|y'=-1]-\ebb[x|y=1]\big)-\big(v_t-u_t\big)\Big\|_2+4\kappa^2|p-p_t||p+p_t-1|.
  \end{multline*}
  Furthermore, there holds that
  \begin{align*}
    & p(1-p)\Big\|\big(\ebb[x'|y'=-1]-\ebb[x|y=1]\big)-\big(v_t-u_t\big)\Big\|_2 \\
    & \leq \Big\|p(1\!-\!p)\big(\ebb[x'|y'\!=\!-1]-\ebb[x|y\!=\!1]\big)-p_t(1\!-\!p_t)(v_t\!-\!u_t)\Big\|_2+\big|p_t(1\!-\!p_t)-p(1\!-\!p)\|\|v_t\!-\!u_t\|_2 \\
    & \leq \Big\|p(1-p)\big(\ebb[x'|y'=-1]-\ebb[x|y=1]\big)-p_t(1-p_t)(v_t-u_t)\Big\|_2+2\kappa|p-p_t||p+p_t-1|.
  \end{align*}
  Combining the above two inequalities and \eqref{concentration-auc-a}, \eqref{concentration-auc-d} together then imply the stated inequality \eqref{concentration-auc-g} with probability $1-\delta$.
  The proof is complete.
\end{proof}

\begin{proof}[Proof of Lemma \ref{lem:approx}]
  It follows from \eqref{tf} that
  \begin{multline}\label{grad-tf}
  \widetilde{F}'(\bw;z)= 2(1-p)\big(x-\ebb[\tilde{x}|\tilde{y}=1]\big)\big(x-\ebb[\tilde{x}|\tilde{y}=1]\big)^\top\bw\ibb_{[y=1]}+\\
  2p(x-\ebb[\tilde{x}|\tilde{y}=-1])(x-\ebb[\tilde{x}|\tilde{y}=-1])^\top\bw\ibb_{[y=-1]}
  +2p(1-p)\big(\ebb[x'|y'=-1]-\ebb[x|y=1]\big)\\
  +2p(1-p)\big(\ebb[x'|y'=-1]-\ebb[x|y=1]\big)\big(\ebb[x'|y'=-1]-\ebb[x|y=1]\big)^\top\bw.
  \end{multline}
  This together with \eqref{grad-hf} shows that
  \begin{multline*}
    \big\| \widetilde{F}'(\bw_t;z_t)-\hat{F}_t'(\bw_t;z_t)\big\|_2 \leq 2\big\|p(1-p)\big(\ebb[x'|y'=-1]-\ebb[x|y=1]\big)-p_t(1-p_t)(v_t-u_t)\big\|_2\\
    + 2\Big\|(1-p)\big(x-\ebb[\tilde{x}|\tilde{y}=1]\big)\big(x-\ebb[\tilde{x}|\tilde{y}=1]\big)^\top-(1-p_t)(x-u_t)(x-u_t)^\top\Big\|_{\mathrm{op}}\|\bw_t\|_2\ibb_{[y_t=1]} \\
    + 2\Big\|p\big(x-\ebb[\tilde{x}|\tilde{y}=-1]\big)\big(x-\ebb[\tilde{x}|\tilde{y}=-1]\big)^\top-p_t(x-v_t)(x-v_t)^\top\Big\|_{\mathrm{op}}\|\bw_t\|_2\ibb_{[y_t=-1]}\\
    + 2\Big\|p(1\!-\!p)\big(\ebb[x'|y'\!=\!-1]-\ebb[x|y\!=\!1]\big)\big(\ebb[x'|y'\!=\!-1]-\ebb[x|y\!=\!1]\big)^\top-p_t(1\!-\!p_t)(v_t\!-\!u_t)(v_t\!-\!u_t)^\top\Big\|_{\mathrm{op}}\|\bw_t\|_2.
  \end{multline*}
  We can apply \eqref{concentration-auc-e}, \eqref{concentration-auc-f}, \eqref{concentration-auc-d} and \eqref{concentration-auc-g} to control each term of the above inequality and derive the stated inequality with probability $1-\delta$.
  The proof is complete.
\end{proof}

\subsection{Boundedness of Iterates}\label{sec:boundedness}

In this subsection, we prove Lemma \ref{lem:boundness} on the almost boundedness of iterates. To this aim, we first establish a recursive inequality showing how $\|\bw_{t+1}-\bw_1^*\|_2^2$ can be controlled by $\|\bw_k-\bw_1^*\|_2^2$ for $k=1,\ldots,t$. Our basic idea is to control $\|\bw_{t+1}-\bw_1^*\|_2^2$ by
\begin{equation}\label{idea-recursive}
  \O(1)\Big(\sum_{k=1}^{t}\eta_k(\phi(\bw_1^*)-\phi(\bw_k))+\sum_{k=1}^{t}\eta_k\big\langle\bw_1^*-\bw_k,\hat{F}_k'(\bw_k;z_k)-\widetilde{F}'(\bw_k;z_k)\big\rangle+\sum_{k=1}^{t}\xi_k\Big),
\end{equation}
where $\{\xi_k\}_k$ is a martingale difference sequence defined in \eqref{bound-p-5}. We apply Lemma \ref{lem:approx} to control $\sum_{k=1}^{t}\eta_k\big\langle\bw_1^*-\bw_k,\hat{F}_k'(\bw_k;z_k)-\widetilde{F}'(\bw_k;z_k)\big\rangle$, and apply Part (b) of Lemma \ref{lem:martingale} to show
with high probability that
$\sum_{k=1}^{t}\xi_k\leq\sum_{k=1}^{t}\eta_k(\phi(\bw_k)-\phi(\bw_1^*))+\widetilde{C}\sum_{k=1}^{t}\eta_k^2\|\bw_k-\bw_1^*\|_2^2$ for a constant $\widetilde{C}>0$. The key observation is that the partial variance $\sum_{k=1}^{t}\eta_k(\phi(\bw_k)-\phi(\bw_1^*))$ can be cancelled out by the term
$\sum_{k=1}^{t}\eta_k(\phi(\bw_1^*)-\phi(\bw_k))$ in \eqref{idea-recursive}.

\begin{proposition}\label{prop:bound-p}
Let $\{\bw_t\}_t$ be produced by \eqref{SAUC} with $\eta_t\leq (2C_1)^{-1}$ and $\eta_{t+1}\leq\eta_t$. We suppose Assumption \ref{ass:self-bounding} holds,
$$
C_5=\sup_k\eta_k\sum_{j=1}^{k-1}\eta_j<\infty,\
C_6=\eta_1\sup_z\widetilde{F}(\bw_1^*,z)+2p(1-p)\Big(7\kappa^2C_4C_5+\eta_1\big(1+2\kappa\|\bw_1^*\|_2+2\kappa^2\|\bw_1^*\|_2^2\big)\Big).
$$
Then for any $\delta\in(0,1)$ and $\rho=\min\{1, (2C_1)^{-1}(\eta_1\|\bw_1^*\|^2_2+C_4C_5)^{-1}C_6\}$, the following inequality holds with probability at least $1-\delta$ simultaneously for all $t=1,\ldots,T$
\begin{multline*}
  \|\bw_{t+1}-\bw_1^*\|_2^2 \leq  \|\bw_1^*\|_2^2+\sum_{k=1}^{t}\frac{2C_{k,\delta}\eta_k(\|\bw_k-\bw_1^*\|_2^2+1)}{\sqrt{k}}
  +\frac{\phi(\bw_1^*)}{C_4C_5}\sum_{k=1}^{t}\eta_k^2\|\bw_k-\bw_1^*\|_2^2\\
  +\frac{2C_6\log(2T/\delta)}{\rho} +2(C_1C_4+A_2)\sum_{k=1}^{t}\eta_k^2,
\end{multline*}
where we introduce the constant
  $$
    C_{k,\delta}=2\kappa^2\big(2+\sqrt{2\log(12k^2/\delta)}\big)\max\big\{C_p+1,4^{-1}(C_p\|\bw_1^*\|_2+3)^2\big\}.
  $$
\end{proposition}
\begin{proof}
  Taking $\bw=\bw_1^*$ in \eqref{osp} gives
  \begin{multline*}
    \|\bw_{t+1}-\bw_1^*\|_2^2 - \|\bw_t-\bw_1^*\|_2^2 \leq 2\eta_t\langle\bw_1^*-\bw_t,\hat{F}_t'(\bw_t,z_t)\rangle + 2\eta_t(\Omega(\bw_1^*)-\Omega(\bw_t)) \\
    + 2\eta_t^2\big(C_1\hat{F}_t(\bw_t;z_t)+C_1\Omega(\bw_t)+A_2\big).
  \end{multline*}
  Taking a summation of the above inequality gives ($\bw_1=0$)
  \begin{align}
     & \|\bw_{t+1}-\bw_1^*\|_2^2 - \|\bw_1^*\|_2^2 = \sum_{k=1}^{t}\big[\|\bw_{k+1}-\bw_1^*\|_2^2-\|\bw_k-\bw_1^*\|_2^2\big]
     \leq 2\sum_{k=1}^{t}\eta_k\langle\bw_1^*-\bw_k,\hat{F}_k'(\bw_k;z_k)\rangle\notag\\
     &+2\sum_{k=1}^{t}\eta_k(\Omega(\bw_1^*)-\Omega(\bw_k))+2\sum_{k=1}^{t}\eta_k^2\big(C_1\hat{F}_k(\bw_k;z_k)+C_1\Omega(\bw_k)+A_2\big)\notag\\
     & \leq 2\sum_{k=1}^{t}\eta_k\langle\bw_1^*-\bw_k,\hat{F}_k'(\bw_k;z_k)\rangle+2\sum_{k=1}^{t}\eta_k(\Omega(\bw_1^*)-\Omega(\bw_k))+2(C_1C_4+A_2)\sum_{k=1}^{t}\eta_k^2,\label{bound-p-1}
  \end{align}
  where the last inequality is due to \eqref{boundness-b}.
  We consider the following decomposition
  \begin{multline}\label{bound-p-2}
    \sum_{k=1}^{t}\eta_k\langle\bw_1^*-\bw_k,\hat{F}_k'(\bw_k;z_k)\rangle = \sum_{k=1}^{t}\eta_k\big\langle\bw_1^*-\bw_k,\hat{F}_k'(\bw_k;z_k)-\widetilde{F}'(\bw_k;z_k)\big\rangle \\
    + \sum_{k=1}^{t}\eta_k\big\langle\bw_1^*-\bw_k,\widetilde{F}'(\bw_k;z_k)-\nabla f(\bw_k)\big\rangle+\sum_{k=1}^{t}\eta_k\langle\bw_1^*-\bw_k,\nabla f(\bw_k)\rangle.
  \end{multline}
  For any $k\in\nbb$, by Lemma \ref{lem:approx} the following inequality holds with probability at least $1-\delta/(4k^2)$ 
  $$
  \big\| \widetilde{F}'(\bw_k;z_k)-\hat{F}_k'(\bw_k;z_k)\big\|_2\leq \frac{2\kappa^2\big(2+\sqrt{2\log(12k^2/\delta)}\big)}{\sqrt{k}}\big(C_p\|\bw_k\|_2+3\big),
  $$
  which together with union bounds and $\sum_{k=1}^{\infty}k^{-2}\leq2$ gives the following inequality with probability $1-\delta/2$ simultaneously for all $k=1,\ldots,\infty$
  \begin{multline}\label{bound-p-12}
  \big\|\widetilde{F}'(\bw_k;z_k)-\hat{F}_k'(\bw_k;z_k)\big\|_2\leq \frac{2\kappa^2\big(2+\sqrt{2\log(12k^2/\delta)}\big)}{\sqrt{k}}\big(C_p\|\bw_k-\bw_1^*\|_2+C_p\|\bw_1^*\|_2+3\big).
  \end{multline}
  It then follows that the following inequality holds with probability at least $1-\delta/2$ simultaneously for all $t=1,\ldots,\infty$
  \begin{align}
    &\sum_{k=1}^{t}\eta_k\big\langle\bw_1^*-\bw_k,\hat{F}_k'(\bw_k;z_k)-\widetilde{F}'(\bw;z_k)\big\rangle \leq \sum_{k=1}^{t}\eta_k\|\bw_k-\bw_1^*\|_2\big\|\hat{F}_k'(\bw_k;z_k)-\widetilde{F}'(\bw_k;z_k)\big\|_2 \notag\\
     &\leq 2\kappa^2\sum_{k=1}^{t}\eta_k\big(2+\sqrt{2\log(12k^2/\delta)}\big) \frac{C_p\|\bw_k-\bw_1^*\|_2^2+(C_p\|\bw_1^*\|_2+3)\|\bw_k-\bw_1^*\|_2}{\sqrt{k}}\notag\\
     & \leq \sum_{k=1}^{t}\eta_kC_{k,\delta}(\|\bw_k-\bw_1^*\|_2^2+1)/\sqrt{k},\label{bound-p-3}
  \end{align}
  where in the last step we have used the Schwartz's inequality 
  \[
    \big(3+C_p\|\bw_1^*\|_2\big)\|\bw_k-\bw_1^*\|_2\leq \|\bw_k-\bw_1^*\|_2^2+(3+C_p\|\bw_1^*\|_2)^2/4.
  \]
  It follows from the convexity of $f$ that
  \begin{equation}\label{bound-p-4}
    \sum_{k=1}^{t}\eta_k\langle\bw_1^*-\bw_k,\nabla f(\bw_k)\rangle \leq \sum_{k=1}^{t}\eta_k\big(f(\bw_1^*)-f(\bw_k)\big).
  \end{equation}
  We now control the last second term of \eqref{bound-p-2} with an application of a concentration inequality for a martingale difference sequence.
  Introduce a sequence of random variables
  \begin{equation}\label{bound-p-5}
    \xi_k:=\eta_k\big\langle\bw_1^*-\bw_k,\widetilde{F}'(\bw_k;z_k)-\nabla f(\bw_k)\big\rangle,\quad k\in\nbb.
  \end{equation}
  It follows from Proposition \ref{lem:unbiased} that $\ebb_{z_k}[\xi_k]=0$ and therefore $\{\xi_k\}_k$ is a martingale difference sequence.
  Analogous to Lemma \ref{lem:self-bounding-F}, we can show
  \begin{equation}\label{bound-p-6}
    \big\|\widetilde{F}'(\bw_k;z_k)\big\|_2^2\leq 16\kappa^2\widetilde{F}(\bw_k,z_k).
  \end{equation}
  Since $\ebb[(\xi-\ebb[\xi])^2]\leq\ebb[\xi^2]$ for any real-valued random variable $\xi$, it then follows that
  \begin{align*}
  &\ebb_{z_k}\Big[\big|\big\langle\bw_1^*-\bw_k,\widetilde{F}'(\bw_k;z_k)-\nabla f(\bw_k)\big\rangle\big|^2\Big] \leq \ebb_{z_k}\Big[\big|\big\langle \bw_1^*-\bw_k,\widetilde{F}'(\bw_k;z_k)\big\rangle\big|^2\Big] \\
  & \leq \|\bw_k-\bw_1^*\|^2_2\ebb_{z_k}\Big[\big\|\widetilde{F}'(\bw_k;z_k)\big\|_2^2\Big] \leq \|\bw_k-\bw_1^*\|^2_2\ebb_{z_k}\big[C_1\widetilde{F}(\bw_k,z_k)\big]
  = C_1f(\bw_k)\|\bw_k-\bw_1^*\|_2^2,
\end{align*}
where we have used the definition of $C_1$ and Proposition \ref{lem:unbiased}.
It then follows that
\begin{align*}
  & \sum_{k=1}^{t}\ebb_{z_k}\big[\big(\xi_k-\ebb_{z_k}[\xi_k]\big)^2\big]
  = \sum_{k=1}^{t}\eta_k^2\ebb_{z_k}\Big[\big|\big\langle\bw_1^*-\bw_k,\widetilde{F}'(\bw_k;z_k)-\nabla f(\bw_k)\big\rangle\big|^2\Big]\\
  & \leq \sum_{k=1}^{t}\eta_k^2\|\bw_k-\bw_1^*\|^2_2\big(C_1\phi(\bw_k)-C_1\phi(\bw_1^*)\big)+\sum_{k=1}^{t}\eta_k^2\|\bw_k-\bw_1^*\|^2_2C_1\phi(\bw_1^*).
\end{align*}
By \eqref{boundness-a}, $C_5=\sup_k\eta_k\sum_{j=1}^{k-1}\eta_j<\infty$ and $f(\bw)\leq\phi(\bw)$, we know
\[
  \eta_k^2\|\bw_k-\bw_1^*\|^2_2 \leq 2\eta_k^2 (\|\bw_k\|_2^2+\|\bw_1^*\|_2^2)\leq 2\eta_k\Big(\eta_k\|\bw_1^*\|_2^2+C_4\eta_k\sum_{j=1}^{k-1}\eta_j\Big)\leq 2\eta_k\big(\eta_1\|\bw_1^*\|_2^2+C_4C_5\big).
\]
Combining the above two inequalities together, we derive
\begin{multline}
  \sum_{k=1}^{t}\ebb_{z_k}\big[\big(\xi_k-\ebb_{z_k}[\xi_k]\big)^2\big]\leq \\
  2C_1\big(\eta_1\|\bw_1^*\|_2^2+C_4C_5\big)\sum_{k=1}^{t}\eta_k\big(\phi(\bw_k)-\phi(\bw_1^*))\big)+C_1\phi(\bw_1^*)\sum_{k=1}^{t}\eta_k^2\|\bw_k-\bw_1^*\|^2_2.\label{bound-p-7}
\end{multline}
According to the convexity of $\widetilde{F}$ established in Proposition \ref{lem:unbiased}, we know
\begin{align*}
   & \xi_k-\ebb_{z_k}[\xi_k] = \eta_k\big\langle\bw_1^*-\bw_k, \widetilde{F}'(\bw_k;z_k)\big\rangle + \eta_k\langle\bw_k-\bw_1^*,\nabla f(\bw_k)\rangle\\
   & \leq \eta_k\big[\widetilde{F}(\bw_1^*;z_k)-\widetilde{F}(\bw_k;z_k)\big] + \eta_k\big(\|\bw_k\|_2+\|\bw_1^*\|_2\big)\big(4p(1-p)\kappa+8p(1-p)\kappa^2\|\bw_k\|_2\big)\\ 
   & \leq \eta_k\widetilde{F}(\bw_1^*;z_k) + 4\eta_kp(1-p)\Big(\kappa\|\bw_k\|_2+\kappa\|\bw_1^*\|_2+2\kappa^2\|\bw_k\|_2^2+2\kappa^2\|\bw_k\|_2\|\bw_1^*\|_2\Big)\\
   & \leq \eta_k\widetilde{F}(\bw_1^*;z_k) + 2\eta_kp(1-p)\Big(\kappa^2\|\bw_k\|_2^2+1+2\kappa\|\bw_1^*\|_2+4\kappa^2\|\bw_k\|_2^2+2\kappa^2\|\bw_k\|_2^2+2\kappa^2\|\bw_1^*\|_2^2\Big)\\
   & \leq \eta_k\widetilde{F}(\bw_1^*;z_k) + 2p(1-p)\Big(7\kappa^2\eta_kC_4\sum_{j=1}^{k-1}\eta_j+\eta_k\big(1+2\kappa\|\bw_1^*\|_2+2\kappa^2\|\bw_1^*\|_2^2\big)\Big)\leq C_6,
\end{align*}
where we have used  the following inequality in the second inequality ($\nabla f(\bw)=2p(1-p)\ebb\big[(1-\bw^\top(x-x'))(x-x')|y=1,y'=-1\big]$)
\begin{equation}\label{bound-grad-f-w}
  \|\nabla f(\bw)\|_2 \leq 4p(1-p)\kappa+8p(1-p)\kappa^2\|\bw\|_2,\quad\forall\bw\in\rbb^d,
\end{equation}
\eqref{boundness-a} and $C_5=\sup_k\eta_k\sum_{j=1}^{k-1}\eta_j<\infty$ in the last inequality.
The above bounds on magnitudes and variances of $\xi_k$ together with Part (b) of Lemma \ref{lem:martingale} (see the Appendix) imply the following inequality with probability $1-\delta/2$
\begin{multline}
  \sum_{k=1}^{t}\xi_k \leq \frac{\rho}{C_6}\bigg(2C_1(\eta_1\|\bw_1^*\|_2^2+C_4C_5)\sum_{k=1}^{t}\eta_k\big(\phi(\bw_k)-\phi(\bw_1^*)\big)+C_1\phi(\bw_1^*)\sum_{k=1}^{t}\eta_k^2\|\bw_k-\bw_1^*\|^2_2\bigg) \\
    +\frac{C_6\log(2/\delta)}{\rho} \leq \sum_{k=1}^{t}\eta_k\big(\phi(\bw_k)-\phi(\bw_1^*)\big)+\frac{\phi(\bw_1^*)}{2C_4C_5}\sum_{k=1}^{t}\eta_k^2\|\bw_k-\bw_1^*\|_2^2+\frac{C_6\log(2/\delta)}{\rho},\label{bound-p-9}
\end{multline}
where we have used the inequality $2C_1\rho(\eta_1\|\bw_1^*\|_2^2+C_4C_5)\leq C_6$.
Plugging \eqref{bound-p-3}, \eqref{bound-p-4} and \eqref{bound-p-9} into \eqref{bound-p-2} gives the following inequality with probability $1-\delta$
\begin{multline*}
  \sum_{k=1}^{t}\eta_k\langle\bw_1^*-\bw_k,\hat{F}_k'(\bw_k;z_k)\rangle \leq  \sum_{k=1}^{t}C_{k,\delta}\eta_k(\|\bw_k-\bw_1^*\|_2^2+1)/\sqrt{k}+\sum_{k=1}^{t}\eta_k\big(f(\bw_1^*)-f(\bw_k)\big)\\
  +\sum_{k=1}^{t}\eta_k\big(\phi(\bw_k)-\phi(\bw_1^*)\big)+\frac{\phi(\bw_1^*)}{2C_4C_5}\sum_{k=1}^{t}\eta_k^2\|\bw_k-\bw_1^*\|_2^2+\frac{C_6\log(2/\delta)}{\rho}.
\end{multline*}
This together with \eqref{bound-p-1} shows the following inequality with probability $1-\delta$
\begin{multline*}
  \|\bw_{t+1}-\bw_1^*\|_2^2 \leq  \|\bw_1^*\|_2^2+\sum_{k=1}^{t}\frac{2C_{k,\delta}\eta_k(\|\bw_k-\bw_1^*\|_2^2+1)}{\sqrt{k}}
  +\frac{\phi(\bw_1^*)}{C_4C_5}\sum_{k=1}^{t}\eta_k^2\|\bw_k-\bw_1^*\|_2^2\\+\frac{2C_6\log(2/\delta)}{\rho}
  +2(C_1C_4+A_2)\sum_{k=1}^{t}\eta_k^2.
\end{multline*}
Note \eqref{bound-p-12} holds simultaneously for all $k=1,\ldots,\infty$. To derive the stated inequality for all $t=1,\ldots,T$, one needs to derive \eqref{bound-p-9} simultaneously for all $k=1,\ldots,T$.
This can be done by replacing $\log(2/\delta)$ in \eqref{bound-p-9} with $\log(2T/\delta)$.
The proof is complete.
\end{proof}

According to the assumption $\sum_{k=1}^{\infty}\eta_k^2<\infty$ and $\sum_{k=1}^{\infty}\eta_k\sqrt{\log k}/\sqrt{k}<\infty$, Proposition \ref{prop:bound-p} essentially implies that
\[
\max_{1\leq k\leq t}\|\bw_k-\bw_1^*\|_2^2\leq \frac{1}{2}\max_{1\leq k\leq t}\|\bw_k-\bw_1^*\|_2^2+\widetilde{C}\log\frac{1}{\delta}
\]
for a $\widetilde{C}>0$, from which we can derive an almost boundedness of $\{\bw_t\}_t$.
We will rigorously show this in the following proof.
\begin{proof}[Proof of Theorem \ref{thm:bound-prob}]
Introduce the set
\begin{multline*}
\Omega_T=\bigg\{(z_1,\ldots,z_T):
  \|\bw_{t+1}-\bw_1^*\|_2^2 \leq  \|\bw_1^*\|_2^2+\sum_{k=1}^{t}\frac{2C_{k,\delta}\eta_k(\|\bw_k-\bw_1^*\|_2^2+1)}{\sqrt{k}}+
  \\
  \frac{\phi(\bw_1^*)}{C_4C_5}\sum_{k=1}^{t}\eta_k^2\|\bw_k-\bw_1^*\|_2^2+\frac{2C_6\log(2T/\delta)}{\rho}
  +2(C_1C_4+A_2)\sum_{k=1}^{t}\eta_k^2\quad\text{for all }t=1,\ldots,T
\bigg\},
\end{multline*}
where $\rho$ is defined in Proposition \ref{prop:bound-p}. Proposition \ref{prop:bound-p} shows that $\text{Pr}(\Omega_T)\geq1-\delta$.
Since $\sum_{t=1}^{\infty}\eta_t\sqrt{\log t}/\sqrt{t}<\infty$ and $\sum_{t=1}^{\infty}\eta_t^2<\infty$, we can find a $t_2\in\nbb$ such that
\begin{equation}\label{bound-1}
  \sum_{k=t_2+1}^{\infty}\frac{2C_{k,\delta}\eta_k}{\sqrt{k}}<1/4\quad\text{and}\quad\sum_{k=t_2+1}^{\infty}\eta_k^2<\frac{C_4C_5}{4\phi(\bw_1^*)}.
\end{equation}
Conditioned on the event $\Omega_T$, we derive the following inequality for all $t=1,\ldots,T$
\begin{align*}
  & \|\bw_{t+1}-\bw_1^*\|_2^2 - \|\bw_1^*\|_2^2\\
  & \leq \sum_{k=1}^{t_2}\frac{2C_{k,\delta}\eta_k\|\bw_k-\bw_1^*\|_2^2}{\sqrt{k}}+\max_{1\leq \tilde{t}\leq T}\|\bw_{\tilde{t}}-\bw_1^*\|_2^2\sum_{k=t_2+1}^{T}\frac{2C_{k,\delta}\eta_k}{\sqrt{k}}
  +\frac{\phi(\bw_1^*)}{C_4C_5}\sum_{k=1}^{t_2}\eta_k^2\|\bw_k-\bw_1^*\|_2^2\\
  & +\frac{\phi(\bw_1^*)\max_{1\leq\tilde{t}\leq T}\|\bw_{\tilde{t}}-\bw_1^*\|_2^2}{C_4C_5}\sum_{k=t_2+1}^{T}\eta_k^2
  + \sum_{k=1}^{t}\frac{2C_{k,\delta}\eta_k}{\sqrt{k}}+\frac{2C_6\log(2T/\delta)}{\rho}
  +2(C_1C_4+A_2)\sum_{k=1}^{t}\eta_k^2\\
  & \leq \sum_{k=1}^{t_2}\frac{2C_{k,\delta}\eta_k\|\bw_k-\bw_1^*\|_2^2}{\sqrt{k}}+\frac{1}{4}\max_{1\leq \tilde{t}\leq T}\|\bw_{\tilde{t}}-\bw_1^*\|_2^2
  +\frac{\phi(\bw_1^*)}{C_4C_5}\sum_{k=1}^{t_2}\eta_k^2\|\bw_k-\bw_1^*\|_2^2\\
  & + \frac{1}{4}\max_{1\leq \tilde{t}\leq T}\|\bw_{\tilde{t}}-\bw_1^*\|_2^2 + \sum_{k=1}^{t}\frac{2C_{k,\delta}\eta_k}{\sqrt{k}}+\frac{2C_6\log(2T/\delta)}{\rho}
  +2(C_1C_4+A_2)\sum_{k=1}^{t}\eta_k^2.
\end{align*}
It then follows the following inequality under the event $\Omega_T$
\begin{multline*}
  \max_{1\leq\tilde{t}\leq T}\|\bw_{\tilde{t}}-\bw_1^*\|_2^2 \leq \|\bw_1^*\|_2^2+\sum_{k=1}^{t_2}\frac{2C_{k,\delta}\eta_k\|\bw_k-\bw_1^*\|_2^2}{\sqrt{k}}+\frac{1}{2}\max_{1\leq \tilde{t}\leq T}\|\bw_{\tilde{t}}-\bw_1^*\|_2^2
  \\
  +\frac{\phi(\bw_1^*)}{C_4C_5}\sum_{k=1}^{t_2}\eta_k^2\|\bw_k-\bw_1^*\|_2^2 + \sum_{k=1}^{T}\frac{2C_{k,\delta}\eta_k}{\sqrt{k}}+\frac{2C_6\log(2T/\delta)}{\rho}
  +2(C_1C_4+A_2)\sum_{k=1}^{T}\eta_k^2,
\end{multline*}
from which we derive the stated inequality with probability $1-\delta$ (notice $\|\bw_k-\bw_1^*\|_2^2\leq 2(\|\bw_1^*\|_2^2+C_4\sum_{j=1}^{k-1}\eta_j)$)
\begin{multline*}
  C_2=2\|\bw_1^*\|_2^2+\sum_{k=1}^{t_2}\frac{8C_7\eta_k\big(\|\bw_1^*\|_2^2+C_4\sum_{j=1}^{k-1}\eta_j\big)}{\sqrt{k}}
  +\\ \frac{4\phi(\bw_1^*)}{C_4C_5}\sum_{k=1}^{t_2}\eta_k^2\big(\|\bw_1^*\|_2^2+C_4\sum_{j=1}^{k-1}\eta_j\big)
  + \sum_{k=1}^{\infty}\frac{4C_7\eta_k}{\sqrt{k}}+\frac{4C_6}{\rho}
  + 4(C_1C_4+A_2)\sum_{k=1}^{\infty}\eta_k^2,
\end{multline*}
where we introduce (notice $C_{k,\delta}\leq C_7\sqrt{\log(T/\delta)}$)
$$
C_7=2\kappa^2\big(2+\sqrt{2\log12+4}\big)\max\Big\{C_p+1,4^{-1}(C_p\|\bw_1^*\|_2+3)^2\Big\}.
$$
The proof is complete.
\end{proof}

\subsection{Proofs for General Convergence Rates}\label{sec:proof-rate}
In this subsection, we prove Theorem \ref{thm:rate} on the probabilistic convergence rates by taking a deduction analogous to the proof of Proposition \ref{prop:bound-p}. The difference is to apply Part (a) of Lemma \ref{lem:martingale} together with the bound of $\|\bw_t\|_2$ established in Theorem \ref{thm:bound-prob} to control $\sum_{k=1}^{t}\xi_k$ in \eqref{bound-p-5}.
\begin{proof}[Proof of Theorem \ref{thm:rate}]
  According to Lemma \ref{lem:approx} followed with union bounds, we know the existence of $\Omega_T^{(1)}$ with $\text{Pr}(\Omega_T^{(1)})\geq1-\delta/3$ such that the following inequality holds with probability $1-\delta/3$ simultaneously for all $t=1,\ldots,T$ conditioned on $\Omega_T^{(1)}$
  $$
  \big\|\widetilde{F}'(\bw_t;z_t)-\hat{F}_t'(\bw_t;z_t)\big\|_2\leq \frac{2\kappa^2\big(2+\sqrt{2\log(9T/\delta)}\big)}{\sqrt{t}}\big(C_p\|\bw_t-\bw_1^*\|_2+3+C_p\|\bw_1^*\|_2\big).
  $$
  It then follows the following inequality conditioned on $\Omega_T^{(1)}$
  \begin{multline}\label{rate-4}
     \sum_{t=1}^{T}\eta_t\big\langle\bw_1^*-\bw_t,\hat{F}_t'(\bw_t;z_t)-\widetilde{F}'(\bw_t;z_t)\big\rangle\ibb_{[\|\bw_t-\bw_1^*\|_2^2\leq C_2\log(6T/\delta)]} \\
     \leq \sum_{t=1}^{T}\eta_t\|\bw_1^*-\bw_t\|_2\Big\|\hat{F}_t'(\bw_t;z_t)-\widetilde{F}'(\bw_t;z_t)\Big\|_2 \ibb_{[\|\bw_t-\bw_1^*\|_2^2\leq C_2\log(6T/\delta)]}
     \leq \widetilde{C}_{T,\delta}\sum_{t=1}^{T}\frac{\eta_t}{\sqrt{t}},
  \end{multline}
  where we introduce
  $$
  \widetilde{C}_{T,\delta}=2\kappa^2\sqrt{C_2}\big(2+\sqrt{2\log(9T/\delta)}\big)\big(C_p\sqrt{C_2}+3+C_p\|\bw_1^*\|_2\big)\log(6T/\delta).
  $$
  Introduce a sequence of random variables
  $$
  \xi_t'=\eta_t\big\langle\bw_1^*-\bw_t,\widetilde{F}'(\bw_t;z_t)-\nabla f(\bw_t)\big\rangle\ibb_{[\|\bw_t-\bw_1^*\|_2^2\leq C_2\log(6T/\delta)]},\quad t=1,\ldots,T.
  $$
  According to Schwartz's inequality, we derive
\begin{align*}
  |\xi_t'| & \leq \eta_t\Big[\|\bw_1^*-\bw_t\|_2^2+4^{-1}\|\widetilde{F}'(\bw_t;z_t)-\nabla f(\bw_t)\|_2^2\Big]\mathbb{I}_{[\|\bw_t-\bw_1^*\|_2^2\leq C_2\log\frac{6T}{\delta}]} \\
   & \leq \eta_t\Big[\|\bw_t-\bw_1^*\|_2^2+2^{-1}\|\widetilde{F}'(\bw_t;z_t)\|_2^2+2^{-1}\|\nabla f(\bw_t)\|_2^2\Big]\mathbb{I}_{[\|\bw_t-\bw_1^*\|_2^2\leq C_2\log\frac{6T}{\delta}]}.
\end{align*}
According to \eqref{grad-tf} and \eqref{bound-grad-f-w}, it is clear that
\begin{align}\label{grad-bound}
  \max\big\{\|\nabla f(\bw)\|_2,\|\widetilde{F}'(\bw;z)\|_2\big\}&\leq 8\kappa^2\|\bw\|_2+\kappa\\
  &\leq 8\kappa^2\|\bw-\bw_1^*\|_2+8\kappa^2\|\bw_1^*\|_2+\kappa.\notag
\end{align}
Therefore, there holds
$$
|\xi_t'|\leq C_8\eta_t\log(6T/\delta),\quad\text{where }C_8=C_2+2(8\kappa^2\|\bw_1^*\|_2+\kappa)^2+128\kappa^4C_2.
$$
It is clear that $\{\xi_t'\}$ is a martingale difference sequence and therefore we can apply Part (a) of Lemma \ref{lem:martingale} in the Appendix to show the existence of $\Omega_T^{(2)}$ with $\text{Pr}(\Omega_T^{(2)})\geq1-\delta/3$ such that the following inequality holds conditioned on $\Omega_T^{(2)}$
\begin{equation}\label{rate-5}
  \sum_{t=1}^{T}\xi_t'\leq C_8\sqrt{2\sum_{t=1}^{T}\eta_t^2\log\frac{3}{\delta}}\log\frac{6T}{\delta}.
\end{equation}
Theorem \ref{thm:bound-prob} implies the existence of $\Omega_T^{(3)}$ with $\text{Pr}(\Omega_T^{(3)})\geq1-\delta/3$ such that $\max_{1\leq\tilde{t}\leq T}\|\bw_{\tilde{t}}-\bw_1^*\|_2^2\leq C_2\log (6T/\delta)$.
According to \eqref{bound-p-2}, \eqref{bound-p-4}, \eqref{rate-4} and \eqref{rate-5}, it is clear that the following inequality holds under the event $\Omega_T^{(1)}\cap \Omega_T^{(2)}\cap \Omega_T^{(3)}$ (note $\xi_t'=\eta_t\big\langle\bw_1^*-\bw_t,\widetilde{F}'(\bw_t;z_t)-\nabla f(\bw_t)\big\rangle$ in this case)
$$
  \sum_{t=1}^{T}\eta_t\langle\bw_1^*-\bw_t,\hat{F}_t'(\bw_t;z_t)\rangle \leq \widetilde{C}_{T,\delta}\sum_{t=1}^{T}\frac{\eta_t}{\sqrt{t}} + C_8\log\frac{6T}{\delta}\sqrt{2\sum_{t=1}^{T}\eta_t^2\log\frac{3}{\delta}} + \sum_{t=1}^{T}\eta_t\big[f(\bw_1^*)-f(\bw_t)\big].
$$
Plugging the above inequality back into \eqref{bound-p-1} and noting $\text{Pr}\big(\Omega_T^{(1)}\cap \Omega_T^{(2)}\cap \Omega_T^{(3)}\big)\geq1-\delta$, we derive the following inequality with probability at least $1-\delta$
\begin{multline*}
  \|\bw_{T+1}-\bw_1^*\|_2^2 -\|\bw_1^*\|_2^2 \leq
  2\sum_{t=1}^{T}\eta_t\big(\phi(\bw_1^*)-\phi(\bw_t)\big)+2(C_1C_4+A_2)\sum_{t=1}^{T}\eta_t^2 \\
  + 2\widetilde{C}_{T,\delta}\sum_{t=1}^{T}\frac{\eta_t}{\sqrt{t}} + 2C_8\log\frac{6T}{\delta}\sqrt{2\sum_{t=1}^{T}\eta_t^2\log\frac{3}{\delta}}.
\end{multline*}
This combined with the convexity of $\phi$ establishes the stated inequality with probability $1-\delta$. The proof is complete.
\end{proof}
\begin{proof}[Proof of Corollary \ref{cor:ave-g}]
  We first prove Part (a). It is clear that the step sizes satisfy \eqref{bound-prob-cond} and therefore Theorem \ref{thm:rate} holds. Part (a) then follows from the standard inequality $\sum_{t=1}^{T}t^{-\theta}\geq (1-\theta)^{-1}(T^{1-\theta}-1),\theta\in(0,1)$. We now turn to Part (b). It is clear that
  $
    \sum_{t=1}^{\infty}\eta_t\log^{\frac{1}{2}}t/\sqrt{t}\leq \eta_1\sum_{t=1}^{\infty}\log^{\frac{1-\beta}{2}}(et)/t<\infty
  $
  and $\sum_{t=1}^{\infty}\eta_t^2<\infty$.
  Part (b) then follows from the inequality
  $
    \sum_{t=1}^{T}\big(t\log^\beta(et)\big)^{-\frac{1}{2}}\geq 2(\sqrt{T}-1)\log^{-\frac{\beta}{2}}(eT).
  $
  The proof is complete.
\end{proof}

\subsection{Proofs for Fast Convergence Rates}\label{sec:proof-rate-fast}
In this subsection, we prove Theorem \ref{thm:strong} on convergence rates for $\phi$ with a quadratic functional growth.
To this aim, we need to introduce some lemmas.
The following lemma provides probabilistic bounds for approximating $\widetilde{F}'(\bw_k;z_k)$ with $\hat{F}_k'(\bw_k;z_k)$ for $\{\bw_k\}$ produced by \eqref{SAUC} with specific step sizes.
\begin{lemma}\label{lem:approx-strong}
  Suppose Assumption \ref{ass:self-bounding} holds.
  Let $\{\bw_t\}_t$ be the sequence produced by \eqref{SAUC} with $\eta_t=\frac{2}{\sigma_\phi t+2\sigma_f+\sigma_\phi t_1}$, where $t_1\geq 4C_1\sigma_\phi^{-1}$. Then, for any $k\leq T$ the following inequality holds with probability $1-\delta$
  \begin{equation}\label{approx-strong}
    \big\| \widetilde{F}'(\bw_k;z_k)-\hat{F}_k'(\bw_k;z_k)\big\|_2\leq C_\delta\sqrt{\log(eT)}/\sqrt{k},
  \end{equation}
  where
  $
    C_\delta:=2\kappa^2\big(2+\sqrt{2\log(3/\delta)}\big)\big(C_p\sqrt{2C_4\sigma_\phi^{-1}}+3\big).
  $
\end{lemma}
\begin{proof}
  Since $t_1\geq 4C_1\sigma_\phi^{-1}$ we know $\eta_t\leq (2C_1)^{-1}$ and therefore Corollary \ref{lem:boundness} holds.
  It follows from the definition of $\eta_t$ that
  \begin{equation}\label{series-a}
    \sum_{k=1}^{t}\eta_k \leq 2\sigma_\phi^{-1}\sum_{k=1}^{t}(k+t_1)^{-1}\leq 2\sigma_\phi^{-1}\log(et).
  \end{equation}
  This together with \eqref{boundness-a} shows
  \begin{equation}\label{w-t-bound}
    \|\bw_t\|_2^2\leq 2C_4\sigma_\phi^{-1}\log(et).
  \end{equation}
   For all $k=1,\ldots,T$, we can then apply Lemma \ref{lem:approx} to derive the following inequality with probability $1-\delta$
  \[
  \big\| \widetilde{F}'(\bw_k;z_k)-\hat{F}_k'(\bw_k;z_k)\big\|_2\leq 2\kappa^2\big(2+\sqrt{2\log(3/\delta)}\big)\big(C_p\sqrt{2C_4\sigma_\phi^{-1}}+3\big)\sqrt{\log(eT)}/\sqrt{k}.
  \]
  The proof is complete with the introduction of $C_\delta$.
\end{proof}

The following lemma plays a fundamental role in our analysis. It shows  that both $\|\bw_t-\bw_t^*\|_2^2$ and a weighted summation of $\phi(\bw_k)-\phi(\bw_k^*)$ can be controlled by a summation of martingale difference sequences. It is established by taking  a weighted summation of the one-step progress inequality \eqref{osp}.
\begin{lemma}\label{lem:step-strong}
Suppose Assumption \ref{ass:self-bounding} and Assumption \ref{ass:strong} hold.
Let $\{\bw_t\}_t$ be the sequence produced by \eqref{SAUC} with $\eta_t=\frac{2}{\sigma_\phi t+2\sigma_f+\sigma_\phi t_1}$ with $t_1\geq 4C_1\sigma_\phi^{-1}$.
Let $\delta\in(0,1)$ and $C_9=16(C_1C_4+A_2)$.
Then the following inequality holds with probability $1-\delta$ for all $t=1,2,\ldots,T$
\begin{multline}\label{step-strong}
\frac{\sum_{k=1}^t(k+t_1+1)(\phi(\bw_k)-\phi(\bw_k^*))}{(t+t_1+1)(t+t_1+2)\sigma_\phi}+\|\bw_{t+1}-\bw_{t+1}^*\|_2^2\leq \frac{(t_1+1)(t_1+2)\|\bw_1-\bw_1^*\|_2^2}{(t+t_1+1)(t+t_1+2)}\\
 +\frac{4\sum_{k=1}^t(k+t_1+1)\xi_k}{(t+t_1+1)(t+t_1+2)\sigma_\phi}+\frac{2\log^2(eT)(2C^2_{\delta/T}+C_9)}{(t+t_1+2)\sigma_\phi^2}.
\end{multline}
\end{lemma}
\begin{proof}
It follows from \eqref{osp} that
\begin{multline*}
    \|\bw_{k+1}-\bw\|_2^2 - \|\bw-\bw_k\|_2^2 \leq 2\eta_k\big\langle\bw-\bw_k,\hat{F}'_k(\bw_k;z_k)-\widetilde{F}'(\bw_k;z_k)\rangle+\\
    +2\eta_k\langle\bw-\bw_k,\widetilde{F}'(\bw_k;z_k)-\nabla f(\bw_k)\rangle+2\eta_k\langle\bw-\bw_k,\nabla f(\bw_k)\rangle
    +2\eta_k\big(\Omega(\bw)-\Omega(\bw_{k})\big)\\
  -\eta_k\sigma_\Omega\|\bw-\bw_{k+1}\|_2^2+2\eta_k^2\big(C_1\hat{F}_k(\bw_k;z_k)+C_1\Omega(\bw_k)+A_2\big).
\end{multline*}
Taking $\bw=\bw_k^*$ in the above inequality and introducing the sequence of random variables $\{\xi_k\}_k$ as
\begin{equation}\label{xi}
    \xi_k=\langle\bw_k^*-\bw_k,\widetilde{F}'(\bw_k;z_k)-\nabla f(\bw_k)\rangle,\quad k=1,2,\ldots,
  \end{equation}
we derive
\begin{multline*}
  (1+\eta_k\sigma_\Omega)\|\bw_{k+1}-\bw_k^*\|_2^2\leq\|\bw_k-\bw_k^*\|_2^2+2^{-1}\eta_k\sigma_\phi\|\bw_k^*-\bw_k\|_2^2+
  2\eta_k\sigma_\phi^{-1}\big\|\hat{F}_k'(\bw_k,z_k)-\widetilde{F}'(\bw_k;z_k)\big\|_2^2\\
  +2\eta_k\xi_k+2^{-1}\eta_k[\phi(\bw_k^*)-\phi(\bw_k)]-3\eta_k\sigma_\phi\|\bw_k-\bw_k^*\|_2^2/2+2\eta_k^2\big(C_1\hat{F}_k(\bw_k;z_k)+C_1\Omega(\bw_k)+A_2\big),
\end{multline*}
where we have used Schwartz's inequality
\[
2\big\langle\bw_k^*-\bw_k,\hat{F}'_k(\bw_k;z_k)-\widetilde{F}'(\bw_k;z_k)\rangle\leq \frac{\sigma_\phi}{2}\|\bw_k^*-\bw_k\|_2^2 + \frac{2}{\sigma_\phi}\big\|\hat{F}'_k(\bw_k;z_k)-\widetilde{F}'(\bw_k;z_k)\big\|_2^2
\]
and the following inequality due to Assumption \ref{ass:strong}
\begin{align*}
  2\langle\bw_k^*-\bw_k,\nabla f(\bw_k)\rangle+2\big(\Omega(\bw_k^*)-\Omega(\bw_{k})\big) &\leq \Big(\frac{1}{2}+\frac{3}{2}\Big)\big(\phi(\bw_k^*)-\phi(\bw_k)\big)\\
  & \leq \frac{1}{2}\Big(\phi(\bw_k^*)-\phi(\bw_k)\Big)-\frac{3}{2}\sigma_\phi\|\bw_k^*-\bw_k\|_2^2.
\end{align*}
It then follows from $\|\bw_{k+1}-\bw_{k+1}^*\|_2\leq\|\bw_{k+1}-\bw_k^*\|_2$ that
\begin{multline}
  \frac{\eta_k(\phi(\bw_k)-\phi(\bw_k^*))}{2(1+\eta_k\sigma_\Omega)}+\|\bw_{k+1}-\bw_{k+1}^*\|_2\leq \frac{1-\eta_k\sigma_\phi}{1+\eta_k\sigma_\Omega}\|\bw_k-\bw_k^*\|_2^2 +
  \\ \frac{2\eta_k\|\hat{F}'_k(\bw_k;z_k)-\widetilde{F}'(\bw_k;z_k)\|_2^2}{(1+\eta_k\sigma_\Omega)\sigma_\phi} + \frac{2\eta_k\xi_k}{1+\eta_k\sigma_\Omega}
  + \frac{2\eta_k^2\big(C_1\hat{F}_k(\bw_k;z_k)+C_1\Omega(\bw_k)+A_2\big)}{1+\eta_k\sigma_\Omega}.\label{strong-a}
\end{multline}
According to the step size choice $\eta_k=\frac{2}{\sigma_\phi k+2\sigma_f+\sigma_\phi t_1}$ and $\sigma_\phi=\sigma_f+\sigma_\Omega$ we know
\[\frac{1-\sigma_\phi\eta_k}{1+\sigma_\Omega\eta_k}\leq \frac{1-\sigma_f\eta_k}{1+\sigma_\Omega\eta_k}=\frac{k+t_1}{k+t_1+2}\quad\text{and}\quad\frac{\eta_k}{1+\sigma_\Omega\eta_k}=\frac{2}{\sigma_\phi(k+t_1+2)}.\]
According to Lemma \ref{lem:approx-strong}, we derive the following inequality with probability at least $1-\delta$ simultaneously for all $k=1,\ldots,T$
\[
\big\| \widetilde{F}'(\bw_k;z_k)-\hat{F}_k'(\bw_k;z_k)\big\|_2\leq C_{\delta/T}\sqrt{\log(eT)}/\sqrt{k}.
\]
Plugging the above two inequalities back into \eqref{strong-a}, we get the following inequality with probability $1-\delta$ for all $k=1,\ldots,T$
\begin{multline*}
  \frac{\phi(\bw_k)-\phi(\bw_k^*)}{\sigma_\phi(k+t_1+2)}+\|\bw_{k+1}-\bw_{k+1}^*\|_2^2\leq \frac{(k+t_1)\|\bw_k-\bw_k^*\|_2^2}{k+t_1+2}+\\
  \frac{4C_{\delta/T}^2\log(eT)}{\sigma_\phi^2k(k+t_1+2)}+\frac{4\xi_k}{\sigma_\phi(k+t_1+2)}+\frac{4\eta_k\big(C_1\hat{F}_k(\bw_k;z_k)+C_1\Omega(\bw_k)+A_2\big)}{\sigma_\phi(k+t_1+2)}.
\end{multline*}
Multiplying both sides with $(k+t_1+2)(k+t_1+1)$ implies the following inequality with probability $1-\delta$ for all $k=1,\ldots,T$
\begin{multline*}
  \frac{(k+t_1+1)(\phi(\bw_k)-\phi(\bw_k^*))}{\sigma_\phi}+(k+t_1+1)(k+t_1+2)\|\bw_{k+1}-\bw_{k+1}^*\|_2^2\\
  \leq (k+t_1)(k+t_1+1)\|\bw_k-\bw_k^*\|_2^2+
  \frac{4C_{\delta/T}^2\log(eT)(k+t_1+1)}{\sigma_\phi^2k}+\frac{4(k+t_1+1)\xi_k}{\sigma_\phi}\\
  +\frac{4\eta_k(k+t_1+1)\big(C_1\hat{F}_k(\bw_k;z_k)+C_1\Omega(\bw_k)+A_2\big)}{\sigma_\phi}.
\end{multline*}
Taking a summation of the above inequality from $k=1$ to $t$ shows the following inequality with probability $1-\delta$ for all $t=1,\ldots,T$
\begin{multline}
  \sigma_\phi^{-1}\sum_{k=1}^t(k+t_1+1)(\phi(\bw_k)-\phi(\bw_k^*))+(t+t_1+1)(t+t_1+2)\|\bw_{t+1}-\bw_{t+1}^*\|_2^2\\
  \leq (t_1+1)(t_1+2)\|\bw_1-\bw_1^*\|_2^2+
  \frac{4C_{\delta/T}^2\log(eT)}{\sigma_\phi^2}\sum_{k=1}^t\frac{k+t_1+1}{k}\\
  +4\sigma_\phi^{-1}\sum_{k=1}^t(k+t_1+1)\xi_k+16\sigma_\phi^{-2}\sum_{k=1}^t\big(C_1\hat{F}_k(\bw_k;z_k)+C_1\Omega(\bw_k)+A_2\big),\label{strong-b}
\end{multline}
where we have used $\eta_k\leq 4/((k+t_1+1)\sigma_\phi)$.
Since $t_1\geq 4C_1\sigma_\phi^{-1}$ we know $\eta_t\leq (2C_1)^{-1}$ and therefore Corollary \ref{lem:boundness} holds.
According to \eqref{series-a} and $\eta_t^{-1}\leq 2^{-1}\sigma_\phi(t+t_1+2)$, we know
\[
\big(\sum_{k=1}^{t}\eta_k\big)\eta_t^{-1}\leq \big(2\sigma_\phi^{-1}\log(et)\big)\Big(2^{-1}\sigma_\phi(t+t_1+2)\Big)=(t+t_1+2)\log(et).
\]
This together with \eqref{boundness-c} implies that
\begin{multline*}
  \sum_{k=1}^t\big(C_1\hat{F}_k(\bw_k;z_k)+C_1\Omega(\bw_k)+A_2\big) \\
  \leq (C_1C_4+A_2)t + C_1C_4\big(\sum_{k=1}^{t}\eta_k\big)\eta_t^{-1}\leq (C_1C_4+A_2)\log(eT)(2t+t_1+2).
\end{multline*}
Plugging the above inequality into \eqref{strong-b} and using $\sum_{k=1}^tk^{-1}\leq \log(eT)$ give the following inequality with probability $1-\delta$
\begin{align*}
  \sigma_\phi^{-1}\sum_{k=1}^t(k+t_1+1)&(\phi(\bw_k)-\phi(\bw_k^*))+(t+t_1+1)(t+t_1+2)\|\bw_{t+1}-\bw_{t+1}^*\|_2^2\\
  &\leq (t_1+1)(t_1+2)\|\bw_1-\bw_1^*\|_2^2+
  \frac{4C_{\delta/T}^2\log(eT)}{\sigma_\phi^2}\big(t+(t_1+1)\log(eT)\big)\\
  &+4\sigma_\phi^{-1}\sum_{k=1}^t(k+t_1+1)\xi_k+C_9\sigma_\phi^{-2}\log(eT)(2t+t_1+2).
\end{align*}
We can get the stated bound by dividing both sides by $(t+t_1+1)(t+t_1+2)$ and noting that
\[
4C_{\delta/T}^2\log(eT)\big(t+(t_1+1)\log(eT)\big)+C_9\log(eT)(2t+t_1+2)\leq 2(t+t_1+1)\log^2(eT)\big(2C^2_{\delta/T}+C_9\big).
\]
The proof is complete.
\end{proof}

To tackle the martingale difference sequence $\{\xi_k\}_k$ in \eqref{step-strong}, we need  to control the magnitudes and variances which are established in the following lemma.
\begin{lemma}\label{lem:var-strong}
  Let Assumption \ref{ass:self-bounding} and Assumption \ref{ass:strong} hold.
  Let $\{\bw_t\}_t$ be the sequence produced by \eqref{SAUC} with $\eta_t=\frac{2}{\sigma_\phi t+2\sigma_f+\sigma_\phi t_1}$, where $t_1\geq 4C_1\sigma_\phi^{-1}$.
  Let $\{\xi_k\}_{k=1}^t$ be defined by \eqref{xi}.
  Then for all $k\leq T$ we have
  \[
  |\xi_k|\leq C_{10}\log(eT)\quad\text{and}\quad\ebb_{z_k}\big[\big(\xi_k-\ebb_{z_k}[\xi_k]\big)^2\big]\leq C_1\phi(\bw_k)\|\bw_k^*-\bw_k\|_2^2,
  \]
  where $C_{10}=34\kappa^2C_4\sigma_\phi^{-1}+2\kappa\|\bw_1^*\|_2+(8\kappa\|\bw_1^*\|_2+1)^2$.
\end{lemma}
\begin{proof}
It follows from the inequality $\|\bw_k-\bw_k^*\|_2\leq\|\bw_k-\bw_1^*\|_2$ and \eqref{grad-bound} that
\begin{align*}
  \langle\bw_k^*-\bw_k,\widetilde{F}'(\bw_k;z_k)-&\nabla f(\bw_k)\rangle \leq \|\bw_1^*-\bw_k\|_2 \big(\|\widetilde{F}'(\bw_k;z_k)\|_2+\|\nabla f(\bw_k)\|_2\big)\\
  & \leq 2\big(\|\bw_1^*\|_2+\|\bw_k\|_2\big)\big(8\kappa^2\|\bw_k\|_2+\kappa\big)\\
  & = 16\kappa^2\|\bw_k\|_2^2+2\kappa\|\bw_1^*\|_2+2\|\bw_k\|_2(8\kappa^2\|\bw_1^*\|_2+\kappa)\\
  & \leq 17\kappa^2\|\bw_k\|_2^2+2\kappa\|\bw_1^*\|_2+(8\kappa\|\bw_1^*\|_2+1)^2\\
  & \leq 34\kappa^2C_4\sigma_\phi^{-1}\log(ek)+2\kappa\|\bw_1^*\|_2+(8\kappa\|\bw_1^*\|_2+1)^2\leq C_{10}\log(eT),
\end{align*}
where we have used \eqref{w-t-bound}.

It is clear from Proposition \ref{lem:unbiased} that $\ebb_{z_k}[\xi_k]=0$ and therefore it follows from $\ebb[(\xi-\ebb[\xi])^2]\leq\ebb[\xi^2]$ for any real-valued random variables $\xi$ that
  \begin{align*}
    \ebb_{z_k}\big[\big(\xi_k-\ebb_{z_k}[\xi_k]\big)^2\big] &= \ebb_{z_k}[\xi_k^2] \leq \ebb_{z_k}\big[\langle \bw_k^*-\bw_k,\widetilde{F}'(\bw_k;z_k)\rangle^2\big] \\
     & \leq \|\bw_k^*-\bw_k\|_2^2\ebb_{z_k}[\|\widetilde{F}'(\bw_k;z_k)\|_2^2]\leq \|\bw_k^*-\bw_k\|_2^2C_1f(\bw_k)\\
     & \leq C_1\phi(\bw_k)\|\bw_k-\bw_k^*\|_2^2.
  \end{align*}
  where we have used
  $\ebb_{z_k}[\|\widetilde{F}'(\bw_k;z_k)\|_2^2]\leq C_1\ebb_{z_k}[\widetilde{F}(\bw_k;z_k)]=C_1f(\bw_k)$ which can be shown analogously to the proof of Lemma \ref{lem:self-bounding-F}.
  The proof is complete.
\end{proof}

We are now ready to prove Theorem \ref{thm:strong}. Our key idea is to apply Part (b) of Lemma \ref{lem:martingale} in the Appendix to show that $\sum_{k=1}^{t}(k+t_1+1)\xi_k$ can be controlled by $\sum_{k=1}^{t}\big(\phi(\bw_k)-\phi(\bw_k^*)\big)(k+t_1+1)$, which can be offset by the first term of \eqref{step-strong}.
Then we can apply the induction strategy to derive the stated bound.

\begin{proof}[Proof of Theorem \ref{thm:strong}]
Since $t_1\geq 32C_1\sigma_\phi^{-1}\log\frac{2T}{\delta}$ and $T\geq2$, we know $t_1\geq 4C_1\sigma_\phi^{-1}$ and therefore Lemmas \ref{lem:approx-strong}, \ref{lem:step-strong}, \ref{lem:var-strong}  hold.
According to Lemma \ref{lem:step-strong}, there exists a set $\Omega_T^{(1)}=\{(z_1,\ldots,z_T)\}$ with $\text{Pr}(\Omega_T^{(1)})\geq1-\delta/2$ such that for all $(z_1,\ldots,z_T)\in\Omega_T^{(1)}$ we have
\begin{multline}\label{strong-0}
\frac{\sum_{k=1}^t(k+t_1+1)(\phi(\bw_k)-\phi(\bw_k^*))}{(t+t_1+1)(t+t_1+2)\sigma_\phi}+\|\bw_{t+1}-\bw_{t+1}^*\|_2^2\leq \frac{(t_1+1)(t_1+2)\|\bw_1^*\|_2^2}{(t+t_1+1)(t+t_1+2)}+\\
 +\frac{4\sum_{k=1}^t(k+t_1+1)\xi_k}{(t+t_1+1)(t+t_1+2)\sigma_\phi}+\frac{2\log^2(eT)(2C^2_{\delta/(2T)}+C_9)}{(t+t_1+2)\sigma_\phi^2}.
\end{multline}
According to Lemma \ref{lem:var-strong}, we know the following inequalities for $k=1,\ldots,t$
\begin{gather*}
|(k+t_1+1)\xi_k|\leq C_{10}(t+t_1+1)\log(eT)\\
  \ebb_{z_k}\big[\big((k+t_1+1)\xi_k-\ebb_{z_k}[(k+t_1+1)\xi_k]\big)^2\big]\leq (k+t_1+1)^2C_1\phi(\bw_k)\|\bw_k^*-\bw_k\|_2^2.
\end{gather*}
Let $\rho\in(0,1]$ to be fixed later.
It then follows from Part (b) of Lemma \ref{lem:martingale} the following inequality with probability $1-\delta/(2T)$
\begin{equation}\label{strong-4}
  \sum_{k=1}^{t}(k+t_1+1)\xi_k \leq \frac{C_1\rho\sum_{k=1}^{t}\phi(\bw_k)(k+t_1+1)^2\|\bw_k^*-\bw_k\|_2^2}{C_{10}(t+t_1+1)\log(eT)}
    +\frac{C_{10}(t+t_1+1)\log(eT)\log\frac{2T}{\delta}}{\rho}.
\end{equation}
By the union bounds of probabilities, we know the existence of $\Omega_T^{(2)}=\{(z_1,\ldots,z_T)\}$ with probability $\text{Pr}(\Omega_T^{(2)})\geq1-\delta/2$ such that \eqref{strong-4} holds
under the event $\Omega_T^{(2)}$ simultaneously for all $t=1,\ldots,T$.
In the remainder of the proof, we always assume that $\Omega_T^{(1)}\cap\Omega_T^{(2)}$ holds (with probability $1-\delta$), and show by induction that $\|\bw_{\tilde{t}+1}-\bw_{\tilde{t}+1}^*\|_2^2\leq C_{T,\delta}/(\tilde{t}+t_1+2)$ for all $\tilde{t}=0,1,\ldots,T-1$ conditioned on $\Omega_T^{(1)}\cap\Omega_T^{(2)}$, where
we introduce
\[
C_{T,\delta}=\max\Big\{2(t_1+1)\|\bw_1^*\|_2^2+\frac{3t_1\phi(\bw^*)}{2\sigma_\phi}+\frac{4\log^2(eT)(2C^2_{\delta/(2T)}+C_9)}{\sigma_\phi^2},\frac{C_{10}t_1\log(eT)}{4C_1}\Big\}
\]
and $\rho=\frac{C_{10}t_1\log(eT)}{4C_1C_{T,\delta}}$. It is clear that $\rho\leq1$.
The case with $\tilde{t}=0$ is clear from the definition of $C_{T,\delta}$. We now show $\|\bw_{t+1}-\bw_{t+1}^*\|_2^2\leq C_{T,\delta}/(t+t_1+2)$ under the induction assumption
\begin{equation}\label{strong-2}
  \|\bw_{\tilde{t}+1}-\bw_{\tilde{t}+1}^*\|_2^2\leq C_{T,\delta}/(\tilde{t}+t_1+2)
\end{equation}
for $\tilde{t}=0,1,\ldots,t-1$.

Plugging the induction assumption \eqref{strong-2} into \eqref{strong-4} gives ($\phi(\bw_k^*)$ is the same for all $k$)
\begin{align*}
   & \sum_{k=1}^{t}(k+t_1+1)\xi_k
   \leq \frac{C_1\rho C_{T,\delta}\sum_{k=1}^{t}\phi(\bw_k)(k+t_1+1)}{C_{10}(t+t_1+1)\log(eT)}
    +\frac{C_{10}(t+t_1+1)\log(eT)\log\frac{2T}{\delta}}{\rho}\\
    & \leq \frac{t_1\sum_{k=1}^{t}\big(\phi(\bw_k)-\phi(\bw_k^*)\big)(k\!+\!t_1\!+\!1)}{4(t\!+\!t_1\!+\!1)}+\frac{t_1\phi(\bw_k^*)\sum_{k=1}^{t}(k\!+\!t_1\!+\!1)}{4(t\!+\!t_1\!+\!1)}
    + \frac{4C_1(t\!+\!t_1\!+\!1)C_{T,\delta}\log\frac{2T}{\delta}}{t_1}\\
    & \leq \frac{t_1\sum_{k=1}^{t}\big(\phi(\bw_k)-\phi(\bw_k^*)\big)(k+t_1+1)}{4(t+t_1+1)}+\frac{3t_1\phi(\bw_k^*)(t+t_1+1)}{16}
    + \frac{4C_1(t+t_1+1)C_{T,\delta}\log\frac{2T}{\delta}}{t_1},
\end{align*}
where the second inequality is due to the definition of $\rho$ and the last inequality is due to $\sum_{k=1}^{t}(k+t_1+1)\leq\frac{3(t+t_1+1)^2}{4}$.

Plugging the above inequality back into \eqref{strong-0} yields the following inequality
\begin{align}
  &\Big(1-\frac{t_1}{t+t_1+1}\Big)\frac{\sum_{k=1}^t(k+t_1+1)(\phi(\bw_k)-\phi(\bw_k^*))}{(t+t_1+1)(t+t_1+2)\sigma_\phi}+\|\bw_{t+1}-\bw_{t+1}^*\|_2^2\notag\\
  &\leq \frac{(t_1+1)\|\bw_1^*\|_2^2}{t+t_1+2}+
 \frac{3t_1\phi(\bw^*)}{4\sigma_\phi(t+t_1+2)}
    + \frac{16C_1C_{T,\delta}\log\frac{2T}{\delta}}{t_1(t+t_1+2)\sigma_\phi}+\frac{2\log^2(eT)(2C^2_{\delta/2T}+C_9)}{(t+t_1+2)\sigma_\phi^2}\notag\\%
 & \leq \frac{(t_1+1)\|\bw_1^*\|_2^2}{t+t_1+2}+
 \frac{3t_1\phi(\bw^*)}{4\sigma_\phi(t+t_1+2)}
    + \frac{C_{T,\delta}}{2(t+t_1+2)}+\frac{2\log^2(eT)(2C^2_{\delta/2T}+C_9)}{(t+t_1+2)\sigma_\phi^2}\label{strong-3},
\end{align}
where the last inequality is due to $t_1\geq 32C_1\sigma_\phi^{-1}\log\frac{2T}{\delta}$.
By the definition of $C_{T,\delta}$, it is clear that the right-hand side of \eqref{strong-3} is less than or equal to $\frac{C_{T,\delta}}{t+t_1+2}$. Therefore, we finish the
induction process and show \eqref{strong-2} for $\tilde{t}=t$.

We now prove the second inequality of \eqref{strong}. It follows from the convexity of $\phi$ and \eqref{strong-3} that
\begin{multline*}
  \phi(\bar{\bw}_t^{(2)})-\phi(\bw_1^*) \leq \Big(\sum_{k=1}^{t}(k+t_1+1)\Big)^{-1}\Big(\sum_{k=1}^{t}(k+t_1+1)\big(\phi(\bw_k)-\phi(\bw^*)\big)\Big)\\
  \leq \frac{2\sigma_\phi(t+t_1+1)^2}{t(t+1)(t+2t_1+3)}\Big((t_1+1)\|\bw_1^*\|_2^2+\frac{3t_1\phi(\bw^*)}{4\sigma_\phi}
    + \frac{C_{T,\delta}}{2}+\frac{2\log^2(eT)(2C^2_{\delta/2T}+C_9)}{\sigma_\phi^2}\Big).
\end{multline*}
The second inequality of \eqref{strong} then follows. 
The proof is complete.
\end{proof}

\section{Conclusion\label{sec:conclusion}}

In this paper, we presented a new stochastic gradient descent method for AUC maximization which can accommodate general penalty terms. Our algorithm can update the model parameter upon receiving individual data with favorable $\O(d)$ space and per-iteration time complexity, making it amenable for streaming data analysis.  We established a high-probability convergence rate $\widetilde{\O}(1/\sqrt{T})$ for the general convex setting, and  a fast convergence $\widetilde{\O}(1/T)$ for the cases of strongly convex regularizers and no regularization term (without strong convexity).

There are several directions for future work. Firstly, we focused on the least square loss and it remains unclear to us on how to  develop similar algorithms for general loss functions.  Secondly,  it would be very interesting to develop stochastic optimization algorithms for AUC maximization under nonlinear models. There are two possible approaches for developing nonlinear models for AUC maximization including the kernel trick and and deep neural networks.  For the approach using the kernel trick, one could use the techniques of random feature  \citep{rahimi2008random} for RBF kernels and then apply the linear model in this paper.  One can easily prove a similar saddle point formulation even for non-convex deep  neural network, and develop stochastic primal-dual stochastic gradient decent algorithms \citep{nemirovski2009robust} for deep AUC maximization models. However, it is not clear on how to establish theoretical guarantees for the convergence of such algorithms as the objective function is generally non-convex.

\bigskip

\noindent{\bf \large Acknowledgements}.  The work of Yiming is supported by the National Science Foundation (NSF) under Grant No. \#1816227. The work of Yunwen is supported by the National Natural Science Foundation of China under Grant No. 61806091
and the Shenzhen Peacock Plan under Grant No. KQTD2016112514355531.

\appendix
\renewcommand{\thesection}{{\Alph{section}}}
\renewcommand{\thesubsection}{\Alph{section}.\arabic{subsection}}
\renewcommand{\thesubsubsection}{\Roman{section}.\arabic{subsection}.\arabic{subsubsection}}
\setcounter{secnumdepth}{-1}
\setcounter{secnumdepth}{3}

\section{Lemmas}
In this section we provide some useful lemmas. Lemma \ref{lem:self-bounding} shows a self-bounding property for smooth and non-negative functions~\citep{nesterov2013introductory}.
\begin{lemma}\label{lem:self-bounding}
  If $h:\rbb^d\to\rbb$ is non-negative and $\beta$-smooth, i.e., $\|\nabla h(\bw)-\nabla h(\tilde{\bw})\|_2\leq\beta\|\bw-\tilde{\bw}\|_2$, then $\|\nabla h(\bw)\|_2^2\leq 2\beta h(\bw)$ for all $\bw\in\rbb^d$.
\end{lemma}

Our discussion is also based on some concentration inequalities. Lemma \ref{lem:hoeffding} is the Hoeffding's inequality for vector-valued random variables~\citep{boucheron2013concentration}.
\begin{lemma}[Hoeffding's inequality]\label{lem:hoeffding}
  Let $Z_1,\ldots,Z_n$ be a sequence of i.i.d. random variables taking values in $\rbb^d$  with $\|Z_i\|_2\leq B$ for every $i$. Then, for any $0<\delta<1$, with probability $1-\delta$ we have
  $$
    \Big\|\frac{1}{n}\sum_{i=1}^{n}\big[Z_i-\ebb[Z_i]\big]\Big\|_2\leq\frac{B}{\sqrt{n}}\Big[2+\sqrt{2\log1/\delta}\Big].
  $$
\end{lemma}

Part (a) of Lemma \ref{lem:martingale} is the Azuma-Hoeffding inequality for martingales with bounded increments \citep{hoeffding1963probability},
and part (b) is a conditional Bernstein inequality using the conditional variance to quantify better the concentration behavior of martingales \citep{zhang2005data}.
\begin{lemma}\label{lem:martingale}
  Let $z_1,\ldots,z_n$ be a sequence of random variables such that $z_k$ may depend on the previous random variables $z_1,\ldots,z_{k-1}$ for all $k=1,\ldots,n$. Consider a sequence of functionals $\xi_k(z_1,\ldots,z_k),k=1,\ldots,n$.
  Let $\sigma_n^2=\sum_{k=1}^{n}\ebb_{z_k}\big[\big(\xi_k-\ebb_{z_k}[\xi_k]\big)^2\big]$ be the conditional variance and $\delta\in(0,1)$.
  \begin{enumerate}[(a)]
    \item Assume that $|\xi_k-\ebb_{z_k}[\xi_k]|\leq b_k$ for each $k$. With probability at least $1-\delta$ we have
    \begin{equation}\label{hoeffding}
      \sum_{k=1}^{n}\xi_k-\sum_{k=1}^{n}\ebb_{z_k}[\xi_k]\leq \Big(2\sum_{k=1}^{n}b_k^2\log\frac{1}{\delta}\Big)^{\frac{1}{2}}.
    \end{equation}
    \item Assume that $\xi_k-\ebb_{z_k}[\xi_k]\leq b$ for each $k$ and $\rho\in(0,1]$. With probability at least $1-\delta$ we have
    \begin{equation}\label{bernstein}
      \sum_{k=1}^{n}\xi_k-\sum_{k=1}^{n}\ebb_{z_k}[\xi_k]\leq \frac{\rho\sigma_n^2}{b}+\frac{b\log\frac{1}{\delta}}{\rho}.
    \end{equation}
  \end{enumerate}
\end{lemma}


\end{document}